\documentclass[10pt,journal,compsoc]{IEEEtran}
\ifCLASSOPTIONcompsoc
  \usepackage[nocompress]{cite}
\else
  \usepackage{cite}
\fi
\usepackage{times}
\usepackage{xcolor}
\usepackage{soul}
\usepackage[utf8]{inputenc}
\usepackage[small]{caption}

\usepackage{lipsum}
\usepackage{amsmath}
\usepackage{amssymb}
\usepackage{amsthm}
\usepackage{bm}
\usepackage{graphicx}
\usepackage{algorithmic}
\usepackage{algorithm}

\usepackage{makecell}
\usepackage{array}
\usepackage{amsthm}
\usepackage{amssymb}
\usepackage{amsmath}
\usepackage{graphicx}  
\usepackage{booktabs}
\usepackage{algorithmic}
\usepackage{algorithm}
\usepackage{helvet}  
\usepackage{multirow}
\usepackage{subfig}
\usepackage{courier}  
\usepackage{url}  

\newtheorem{theorem}{Theorem}
\newtheorem{lemma}{Lemma}

\ifCLASSINFOpdf
\else
\fi
\begin{document}
%

\title{Sparse PCA via $\ell_{2,p}$-Norm Regularization \\
for Unsupervised Feature Selection}

\author{\author{Zhengxin Li,
        Feiping Nie*,
        Jintang Bian,
        and Xuelong Li,~\IEEEmembership{Fellow,~IEEE}
\IEEEcompsocitemizethanks{
\IEEEcompsocthanksitem Corresponding author: Feiping Nie. \protect
\IEEEcompsocthanksitem Z. Li was with the College
of Equipment Management and UAV Engineering, Air Force Engineering University, Xi'an 710051, Shaanxi, P. R. China.
He is currently a postdoctor in the School of Computer Science
and Center for OPTical IMagery Analysis and Learning (OPTIMAL),
Northwestern Polytechnical University. \protect
E-mail: {zhengxinli@nwpu.edu.cn}
\IEEEcompsocthanksitem F. Nie, J. Bian and X. Li were with
the School of Computer Science
and Center for OPTical IMagery Analysis and Learning (OPTIMAL),
Northwestern Polytechnical University, Xi'an 710072, Shaanxi, P. R. China.\protect
E-mail: {feipingnie@gmail.com; bianjintang@gmail.com; xuelong\_li@nwpu.edu.cn}
}}}


%
%

\markboth{Journal of \LaTeX\ Class Files,~Vol.~14, No.~8, August~2015}%
{Shell \MakeLowercase{\textit{et al.}}: Bare Demo of IEEEtran.cls for Computer Society Journals}
%



\IEEEtitleabstractindextext{%
\begin{abstract}

In the field of data mining,
how to deal with high-dimensional data is an inevitable problem.
Unsupervised feature selection has attracted more and more attention because it does not rely on labels.
The performance of spectral-based unsupervised methods depends on the quality of constructed similarity matrix,
which is used to depict the intrinsic structure of data.
However, real-world data contain a large number of noise samples and features, making the similarity matrix constructed by original data cannot be completely reliable.
Worse still,
the size of similarity matrix expands rapidly as the number of samples increases,
making the computational cost increase significantly.
Inspired by principal component analysis,
we propose a simple and efficient unsupervised feature selection method,
by combining reconstruction error with $\ell_{2,p}$-norm regularization.
The projection matrix, which is used for feature selection,
is learned by minimizing the reconstruction error under the sparse constraint.
Then, we present an efficient optimization algorithm to solve the proposed unsupervised model,
and analyse the convergence and computational complexity of the algorithm theoretically.
Finally,
extensive experiments on real-world data sets demonstrate the effectiveness of our proposed method.

\end{abstract}

}

\maketitle

\IEEEdisplaynontitleabstractindextext

%
\IEEEpeerreviewmaketitle

\IEEEraisesectionheading{\section{Introduction}\label{sec:introduction}}

\IEEEPARstart{W}{ith} the rapid development of information technology, high-dimensional data exist almost everywhere in all walks of life, such as
weather forecast~\cite{Ma8440052},
financial transaction analysis~\cite{He2020Data},
geological prospecting~\cite{Dentith2020},
image search~\cite{Deng2019},
text mining~\cite{Ali2019Fuzzy},
bioinformatics~\cite{Luo8097006}, \emph{etc}.
Unfortunately,
the curse of dimensionality seriously restricts many practical applications.
To solve this problem,
feature selection is used to
reduce the dimension by finding a relevant feature subset of data~\cite{2003An}.
The advantages of feature selection mainly include:
improving the performance of data mining tasks,
reducing computational cost,
improving the interpretability of data.
Therefore, feature selection has become a necessary prerequisite for many data mining tasks, such as
pattern recognition~\cite{Zhang2019Fuzzy},
clustering~\cite{Idrobo2019clustering},
classification~\cite{Kayabol2020classification},
similarity retrieval~\cite{Wu2020retrieval}, \emph{etc}.

Based on whether data labels are available,
feature selection can be divided into supervised and unsupervised methods~\cite{Yu2004Eficient}.
Supervised feature selection utilizes the correlation between features and labels to find discriminative features.
However, obtaining labels is expensive, or even impractical in many applications.
Thus, unsupervised feature selection has attracted a lot of attention,
because it does not rely on labels.
In the paper, we propose a new method for unsupervised feature selection.
The main contributions are summarized as follows:
\begin{itemize}
  \item A new unsupervised model is proposed to perform feature selection.
The sparse projection matrix is learned by minimizing the reconstruction error of data.
  \item An optimization algorithm is presented to solve the proposed model. We prove the convergence of the algorithm,
and evaluate its computational complexity,
which is linear to the number of samples.
  \item Extensive experiments on real-world data sets demonstrate the effectiveness of our proposed method.
\end{itemize}

The rest paper is organized as follows.
In Section~\ref{sec:Background},
we give a brief review of the related work and introduce some notations and definitions.
In Section~\ref{sec:proposed method},
we propose a new unsupervised feature selection model.
In Section~\ref{sec:Optimization algorithm},
the optimization algorithm is presented to solve the proposed model.
In Section~\ref{sec:Discussion},
we discuss the convergence and computational complexity of the optimization algorithm.
In Section~\ref{sec:experiments},
experiments are implemented to evaluate the effectiveness of the proposed method.
Finally,
we provide the conclusion in Section~\ref{sec:conclusions}.

\section{Background}\label{sec:Background}

\subsection{Related work}\label{sec:related work}

The techniques of unsupervised feature selection
can be divided into three types~\cite{Li2014Clustering}:
filter, wrapper and embedded methods.

Filter methods~\cite{Lazar2012A} are independent of the data mining tasks.
They are usually intuitive and computationally efficient.
LapScore (Laplacian Score)~\cite{LaplacianScore2005} is
one of the most classic filter methods.
It calculates the score for each feature independently,
according to its ability to preserve the intrinsic structure of original data.
Then, all the features are ranked by the scores.
Because each feature is evaluated independently,
it may work well on binary-cluster problems,
but are very likely to fail in multi-cluster cases~\cite{Deng2010Unsupervised}.

Wrapper methods~\cite{Kabir2008A} combine feature selection with the data mining tasks.
The mining algorithm is utilized to evaluate the effectiveness of selected features.
The result of feature selection performs well in the mining task.
However, wrapper methods are usually computationally expensive and weak in generalization.

Embedded methods~\cite{Chong2016Feature} integrate feature selection into model learning.
Since there is no need to evaluate feature subsets,
they are more efficient than wrapper methods~\cite{Li2016Feature}.
Thus, embedded methods have gradually become a hotspot, and many representative methods keep emerging, such as
MCFS (Multi-Cluster Feature Selection)~\cite{Deng2010Unsupervised},
UDFS (Unsupervised Discriminative Feature Selection)~\cite{YangUDFS2011},
EUFS (Embedded Unsupervised Feature Selection)~\cite{Wang2015Embedded},
DGUFS (Dependence Guided Unsupervised Feature Selection)~\cite{GuoDGUFS2018},
SOGFS (Structured Optimal Graph Feature Selection)~\cite{SOGFS2019}
and RNE (Robust Neighborhood Embedding Feature Selection)~\cite{RNE2020},
\emph{etc}.

MCFS selects features by using spectral regression with $\ell_{1}$-norm regularization, so that the multi-cluster structure of original data can be preserved.
UDFS selects the discriminative features by joint discriminative analysis and $\ell_{2,1}$-norm minimization.
EUFS embeds unsupervised feature selection into a clustering algorithm via sparse learning. $\ell_{2,1}$-norm is applied on the cost function to reduce the effects of noise.
DGUFS enhances the interdependence among original data, cluster labels, and selected features.
SOGFS conducts feature selection and local structure learning simultaneously, so that the similarity matrix can be determined adaptively.
RNE selects features by calculating feature weight matrix through locally linear embedding algorithm, and ultilizing $\ell_{1}$-norm to minimize its reconstruction error.

Most embedded methods, even including LapScore,
utilize spectral analysis and manifold learning to select discriminative features.
They usually build a similarity matrix
to depict the intrinsic structure of original data.
However, real-world data contain a large number of noise samples and features, making the similarity matrix constructed by original data cannot be completely reliable.
Worse still,
the size of similarity matrix expands rapidly as the number of samples increases,
making the computational cost increase significantly.
Inspired by principal component analysis~\cite{2002Principal},
we propose a simple and efficient unsupervised feature selection method from a new perspective.

\subsection{Notations and definitions}\label{sec:Notations and definitions}

We first introduce some notations and definitions that will be used throughout the paper.
Given a matrix $M\in \mathbb{R}^{n \times d}$,
the ($i,j$)-th element of $M$ is denoted by $m_{ij}$,
its $i$-th row, $j$-th column are denoted by $m^i$, $m_j$ respectively.
The transpose of $M$ is denoted by $M^T$.
The trace of $M$ is denoted by ${\rm Tr}(M)$.
The $\ell_{2,p}$-norm is defined as:
\begin{equation}\label{eq:l2,p norm}
\begin{split}
     \left\|{\rm M}\right\|_{2,p}=\left(\sum_{i=1}^{n}  {\left( \sum_{j=1}^{d} m_{ij}^2 \right) }^\frac{p}{2} \right)^\frac{1}{p}
     = \left(\sum_{i=1}^{n} \left\| m^i \right\|_2^p \right)^\frac{1}{p}, \  p>0
\end{split}
\end{equation}

When $p \geq 1$, since it satisfies the basic norm conditions,
$\ell_{2,p}$-norm is a valid norm.
However, when $0<p<1$, $\ell_{2,p}$ is not a valid norm.
For convenience, we still call them norms in the paper.

\section{Unsupervised feature selection model}\label{sec:proposed method}

Supposing a data set \{$x_1,x_2,\ldots,x_n$\} contains $n$ data points $x_i \in \mathbb{R}^{d \times 1}$,
$X\in \mathbb{R}^{n \times d}$ denotes the data matrix.
Without loss of generality, we assume that all the data points are centralized:
\begin{equation}\label{eq:data centralized}
\begin{split}
    \sum_{i=1}^n x_i=0
\end{split}
\end{equation}

Supposing we explore principal component analysis for dimension reduction,
the new coordinate system formed by principal components is:
\begin{equation}\label{eq:coordinate system}
\begin{split}
     & \{w_1,w_2,\ldots,w_d\}, \quad  w_i \in \mathbb{R}^{d \times 1} \\
     & s.t.  \quad \|w_i\|=1, \quad w^T_iw_j=0 \ (i\neq j)
\end{split}
\end{equation}

If we want to reduce the dimension of the data points from $d$ to  $m$ $(m<d)$,
some coordinates in the coordinate system should be discarded.
Then, the new coordinate system $W \in \mathbb{R}^{d \times m}$ is:
\begin{equation}\label{eq:new coordinate system}
\begin{split}
     & W=\{w_1,w_2,\ldots,w_{m}\}, \quad  w_i \in \mathbb{R}^{d \times 1} \\
     & s.t.  \quad \|w_i\|=1, \quad w^T_i w_j=0 \ (i\neq j)
\end{split}
\end{equation}

Thus, the projection of the data point $x_i$ in the new coordinate system is:
\begin{equation}\label{eq:projection result 1}
\begin{split}
     z_i=\{z_{i1},z_{i2},\ldots,z_{im}\}^T, \quad z_{ij}=w^T_j x_i
\end{split}
\end{equation}
where $z_{ij}$ is the $j$th-dimension coordinate of $x_i$ in the low dimensional coordinate system.
Eq.~\eqref{eq:projection result 1} can be rewritten as
\begin{equation}\label{eq:projection result 2}
\begin{split}
     z_i=W^Tx_i
\end{split}
\end{equation}

If we reconstruct $x_i$ with $z_i$, the original data point can be recovered as:
\begin{equation}\label{eq:reconstruct data}
\begin{split}
     \hat{x_i}=\sum_{j=1}^{m} {z_{ij}w_j} = Wz_{i}
\end{split}
\end{equation}

For the entire data set,
the sum of the error between each original data point $x_i$ and its reconstructed point $\hat{x_i}$ is:
\begin{equation}\label{eq:reconstructed error 1}
\begin{split}
     \sum_{i=1}^{n}\left\|\hat{x}_{i}-x_{i}\right\|_2^2
\end{split}
\end{equation}

We can substitute Eq.~\eqref{eq:reconstruct data} into Eq.~\eqref{eq:reconstructed error 1}:
\begin{equation}\label{eq:reconstructed error 2}
\begin{split}
     \sum_{i=1}^{n}\left\|Wz_{i}-x_{i}\right\|_2^2
\end{split}
\end{equation}

According to the property of $\ell_2$-norm,
Eq.~\eqref{eq:reconstructed error 2} can be further expanded to:
\begin{equation}\label{eq:equivalent transformation 1}
\begin{split}
    \sum_{i=1}^{n}\left(W z_{i}\right)^{T}\left(W z_{i}\right)-2 \sum_{i=1}^{n}\left(W z_{i}\right)^{T} x_{i}+\sum_{i=1}^{n} x_{i}^{T} x_{i}
\end{split}
\end{equation}

Due to $w^T_iw_j=0 \ (i\neq j)$, we can get $W^TW=I$.
Then, Eq.~\eqref{eq:equivalent transformation 1} can be converted to the following equation:
\begin{equation}\label{eq:equivalent transformation 2}
\begin{split}
    \sum_{i=1}^{n} z_{i}^{T} z_{i}-2 \sum_{i=1}^{n} z_{i}^{T} W^{T} x_{i}+\sum_{i=1}^{n} x_{i}^{T} x_{i}
\end{split}
\end{equation}

According to Eq.~\eqref{eq:projection result 2},
the above equation can be rewritten as
\begin{equation}\label{eq:equivalent transformation 3}
\begin{split}
    & \sum_{i=1}^{n} z_{i}^{T} z_{i}-2 \sum_{i=1}^{n} z_{i}^{T} z_{i}+\sum_{i=1}^{n} x_{i}^{T} x_{i} \\
    & =-\sum_{i=1}^{n} z_{i}^{T} z_{i}+\sum_{i=1}^{n} x_{i}^{T} x_{i}
\end{split}
\end{equation}

According to Eq.~\eqref{eq:data centralized}, Eq.~\eqref{eq:projection result 2} and the properties of matrix trace, we can get
\begin{equation}\label{eq:equivalent transformation 4}
\begin{split}
    \sum_{i=1}^{n} z_{i}^{T} z_{i} & = \operatorname{Tr}\left(W^{T}\left(\sum_{i=1}^{n} x_{i} x_{i}^{T}\right) W\right) \\
    & =\operatorname{Tr}(W^{T} X^{T}X W)
\end{split}
\end{equation}

We can further substitute Eq.~\eqref{eq:equivalent transformation 4} into Eq.~\eqref{eq:equivalent transformation 3}.
Thus, Eq.~\eqref{eq:reconstructed error 1} is equivalent to
\begin{equation}\label{eq:equivalent transformation 5}
\begin{split}
    -\operatorname{Tr}(W^{T} X^{T}X W)+\sum_{i=1}^{n} x_{i}^{T} x_{i}
\end{split}
\end{equation}

Principal component analysis requires that the reconstruction error should be minimal.
Thus, the objective function is
\begin{equation}\label{eq:objective function of PCA 1}
\begin{split}
    \min_{W^T W=I} -\operatorname{Tr}(W^{T} X^{T}X W)+\sum_{i=1}^{n} x_{i}^{T} x_{i}
\end{split}
\end{equation}

For a given data set, $\sum_{i=1}^{n} x_{i}^{T} x_{i}$ is a constant,
which has no impact on the minimization of the objective function.
Then, Eq.~\eqref{eq:objective function of PCA 1} can be rewritten as
\begin{equation}\label{eq:objective function of PCA 2}
\begin{split}
    \min_{W^T W=I} -\operatorname{Tr}(W^{T} X^{T}X W)
\end{split}
\end{equation}

For the general case that the data points are not centralized,
Eq.~\eqref{eq:objective function of PCA 2} can be rewritten as
\begin{equation}\label{eq:objective function of PCA 3}
\begin{split}
    \min_{W^T W=I} -\operatorname{Tr}(W^T S_t W)
\end{split}
\end{equation}
where $S_t= X^THX$ is the total scatter matrix.
$H$ is the centering matrix:
\begin{equation}\label{eq:centering matrix}
\begin{split}
    H = I_n - \frac{1}{n} \textbf{1}\textbf{1}^T
\end{split}
\end{equation}

As we all known,
$\ell_{2,0}$-norm is the most suitable for feature selection.
For the sake of feature selection,
we add a regularization term to the objective function of Eq.~\eqref{eq:objective function of PCA 3}:
\begin{equation}\label{eq:problem definition 0}
\begin{split}
    & \min_{W^T W=I} -Tr(W^T S_t W)+\gamma \| W \|_{2,0}
\end{split}
\end{equation}
where $\gamma>0$ is a regularization parameter.
The regularization term can make the projection matrix $W$ be sparse on the row vectors,
so as to complete the task of feature selection.
Unfortunately,
it is difficult to solve $\ell_{2,0}$-norm problem directly.
Because $\ell_{2,p}$-norm $(0<p\leq 1)$
is a reasonable choice to approximate $\ell_{2,0}$-norm in the feature selection task~\cite{SOGFS2019},
we can replace $\ell_{2,0}$-norm with $\ell_{2,p}$-norm.
Thus, Eq.~\eqref{eq:problem definition 0} can be rewritten as
\begin{equation}\label{eq:problem definition}
\begin{split}
    & \min_{W^T W=I} -Tr(W^T S_t W)+\gamma \| W \|_{2,p}^p,\ 0<p\leq 1
\end{split}
\end{equation}

In the process of minimizing the objective function of Eq.~\eqref{eq:problem definition},
$\gamma \| W \|_{2,p}^p$ favors a small number of nonzero row vector $w^i$.
The projection matrix $W$ should satisfy the following two constraints:
it is sparse in the row vectors;
the reconstruction error of all data points should be as small as possible,
which is just the optimization direction of principal component analysis.
For simplicity,
we denote the proposed method as SPCAFS (Sparse Principal Component Analysis for Feature Selection).

\section{Optimization algorithm }\label{sec:Optimization algorithm}

In this section,
we present the optimization algorithm to solve problem~\eqref{eq:problem definition}.
According to the definition of $\ell_{2,p}$-norm,
problem~\eqref{eq:problem definition} can be rewritten as
\begin{equation}\label{eq:problem transition 1}
\begin{split}
    & \min_{W^T W=I} -\operatorname{Tr}(W^T S_t W)+
   \gamma \sum_{i=1}^{d} \left\| w^i \right\|_2^p   \\
\end{split}
\end{equation}
where $w^i\in \mathbb{R}^{m \times 1}$ is the $i$-th row vector of $W$.
Since $\left\| w^i \right\|_2^p$ can be zero in theory,
Eq.~\eqref{eq:problem transition 1} may be non-differentiable.
To avoid this case,
we replace $\left\| w^i \right\|_2^p$ with $\left(w^{iT}w^{i}\right)^\frac{p}{2}$.
Further, it is regularized as
\begin{equation}\label{eq:problem transition 2}
\begin{split}
    \left(w^{iT}w^{i}\right)^\frac{p}{2} \rightarrow \left(w^{iT}w^{i}+\epsilon\right)^\frac{p}{2}
\end{split}
\end{equation}
where $\epsilon$ is a sufficiently small constant.
Then, Eq.~\eqref{eq:problem transition 1} can be equivalent to
\begin{equation}\label{eq:problem transition 3}
\begin{split}
    & \min_{W^T W=I} -\operatorname{Tr}(W^T S_t W)
    +\gamma  \sum_{i=1}^{d} \left(w^{iT}w^{i}+\epsilon\right)^\frac{p}{2}
\end{split}
\end{equation}

\begin{theorem}\label{theorem:1}
 The solution to problem~\eqref{eq:problem transition 3},
 i.e. $W \in \mathbb{R}^{d \times m}$, will contain at least $m$ non-zero rows.
\end{theorem}

\begin{proof}
 According to the constraint of problem~\eqref{eq:problem transition 3},
 any feasible solution $W \in \mathbb{R}^{d \times m}$ should satisfy $W^TW = I_m$.
 Since $I_m \in \mathbb{R}^{m \times m}$, the rank of $W$ is $m$.
 Therefore,  $W$ contains at least $m$ non-zero rows.
\end{proof}

The Lagrangian function of problem~\eqref{eq:problem transition 3} is
\begin{equation}\label{eq:Lagrangian function}
\begin{split}
    \mathcal{L}(W,\Lambda) = & -\operatorname{Tr}(W^T S_t W)+\gamma \sum_{i=1}^{d} \left(w^{iT}w^{i}+\epsilon\right)^\frac{p}{2} \\
  & + \operatorname{Tr}(\Lambda(W^T W-I))
\end{split}
\end{equation}
where $\Lambda$ is the Lagrangian multiplier.
We take the derivative of Eq.~\eqref{eq:Lagrangian function} with respect to $W$,
and set its value equal to zero.
Then, we can get
\begin{equation}\label{eq:the derivative of Lagrangian function}
\begin{split}
    \frac{\partial \mathcal{L}(W,\Lambda)}{\partial W}  = & -S_t W+\gamma GW + W \Lambda = 0
\end{split}
\end{equation}
where $G \in \mathbb{R}^{d \times d}$ is a diagonal matrix, and the $i$-th diagonal element is defined as
\begin{equation}\label{eq:G definition}
\begin{split}
    g_{ii}=\frac{p}{2} \left( w^{iT}w^{i}+\epsilon \right)^\frac{p-2}{2}
\end{split}
\end{equation}

It is worth noting that $G$ still depends on $W$.
That is, $W$ cannot be directly calculated from Eq.~\eqref{eq:the derivative of Lagrangian function}.
Thus, we utilize the following alternate optimization method to  calculate $W$, $G$ iteratively.

\textbf{Fix $G$ update $W$}.

When $G$ is fixed, it is easily to prove that solving Eq.~\eqref{eq:the derivative of Lagrangian function} is equivalent to solving
\begin{equation}\label{eq:update W}
\begin{split}
    & \min_{W^T W=I} -\operatorname{Tr}(W^T S_t W)+\gamma \operatorname{Tr}(W^T G W)
\end{split}
\end{equation}

The optimal $W$ of Eq.~\eqref{eq:update W} is formed by the $m$ eigenvectors of $(-S_t+\gamma G)$,
corresponding to the $m$ smallest eigenvalues.

\textbf{Fix $W$ update $G$}.

When $W$ is fixed, we can easily calculate $G$ by Eq.~\eqref{eq:G definition}.

Based on the above analysis,
the optimization algorithm to solve problem~\eqref{eq:problem transition 3} is summarized in Algorithm~\ref{algorithm:SPCAFS}.

\begin{algorithm}[h]
\caption{The algorithm to solve the problem~\eqref{eq:problem transition 3}.}\label{algorithm:SPCAFS}
\begin{algorithmic}[1]
\REQUIRE Data matrix $X \in \mathbb{R}^{n \times d}$, reduced dimension $m$,
regularization parameter $\gamma$, a sufficiently small constant $\epsilon$. \\
\ENSURE $h$ features of the data set. \\
\STATE Initialize $S_t= X^THX$ and $G=I$.
\REPEAT
\STATE  Update $W \in \mathbb{R}^{d \times m}$. The columns of $W$ are the $m$ eigenvectors of $(-S_t+\gamma G)$, corresponding to the $m$ smallest eigenvalues.
\STATE  Update $G$. The $i$-th element of $G$ is defined by Eq.~\eqref{eq:G definition}.
\UNTIL{converge}
\STATE Sort $\|w^i\|_2$ ($i = 1, 2, \dots, d$) in descending order,
and select the top $h$ ranked features.
\end{algorithmic}
\end{algorithm}

\section{Discussion }\label{sec:Discussion}

\subsection{Convergence analysis}\label{sec:Convergence analysis}

The convergence of Algorithm~\ref{algorithm:SPCAFS}
can guarantee that we can find a locally optimal solution of problem~\eqref{eq:problem transition 3}.
Obviously, the converged solution satisfies KKT condition.
To prove the convergence, we first introduce the following lemma.
Please refer to~\cite{SOGFS2019} for the detailed proof.

\begin{lemma}\label{lemma nie}
When $0<p\leq 1$, for any positive real number $u$ and $v$, the following inequality holds:
\begin{equation}\label{eq:lemma 1}
\begin{split}
    u^\frac{p}{2}- \frac{p}{2} \frac{u}{v^\frac{2-p}{2}}
    \leq v^\frac{p}{2}-\frac{p}{2} \frac{v}{v^\frac{2-p}{2}}
\end{split}
\end{equation}
\end{lemma}


\begin{theorem}\label{theorem:2}
 When we calculate $W$ according to Algorithm~\ref{algorithm:SPCAFS},
 updated $W$ will decrease the objective value of problem~\eqref{eq:problem transition 3} until converge.
\end{theorem}
\begin{proof}
  Supposing the current updated $W$ is denoted by $\hat{W}$,
we can easily derive the following inequality
\begin{equation}\label{eq:proof theorem 1}
\begin{split}
  & -\operatorname{Tr}(\hat{W}^T S_t \hat{W})+\gamma \operatorname{Tr}(\hat{W}^T G \hat{W}) \leq \\
  & -\operatorname{Tr}(W^T S_t W)+\gamma \operatorname{Tr}(W^T G W)
\end{split}
\end{equation}

We add the following item to the both sides of Eq.~\eqref{eq:proof theorem 1}.
\begin{equation}\label{eq:proof theorem 2}
\begin{split}
 \gamma \sum_{i=1}^{d} \frac{p}{2} \frac{\epsilon}{ \left( w^{iT} w^i+\epsilon \right)^\frac{2-p}{2}}
\end{split}
\end{equation}

Then, we can get
\begin{equation}\label{eq:proof theorem 3}
\begin{split}
&-\operatorname{Tr}(\hat{W}^T S_t\hat{W})+\gamma\operatorname{Tr}(\hat{W}^TG\hat{W})
  + \sum_{i=1}^{d} \frac{\gamma p\epsilon}{ 2\left( w^{iT} w^i+\epsilon \right)^\frac{2-p}{2}} \leq \\
&-\operatorname{Tr}(W^T S_t W)+\gamma \operatorname{Tr}(W^T G W)
   +\sum_{i=1}^{d} \frac{\gamma p\epsilon}{ 2\left( w^{iT} w^i+\epsilon \right)^\frac{2-p}{2}}
\end{split}
\end{equation}

According to Eq.~\eqref{eq:G definition}, we substitute $G$ into Eq.~\eqref{eq:proof theorem 3}:
\begin{equation}\label{eq:proof theorem 4}
\begin{split}
 &  -\operatorname{Tr}(\hat{W}^T S_t \hat{W})+
 \gamma \sum_{i=1}^{d} \frac{p}{2} \frac{\hat{w}^{iT} \hat{w}^i+\epsilon}{ \left( w^{iT} w^i+\epsilon \right)^\frac{2-p}{2} } \leq \\
 & -\operatorname{Tr}(W^T S_t W)+
 \gamma \sum_{i=1}^{d} \frac{p}{2} \frac{w^{iT} w^i+\epsilon}{ \left( w^{iT} w^i+\epsilon \right)^\frac{2-p}{2}}
\end{split}
\end{equation}

Setting $u=\hat{w}^{iT} \hat{w}^i+\epsilon$, $v=w^{iT} w^i+\epsilon$,
according to Lemma~\ref{lemma nie}, we have the inequality:
\begin{equation}\label{eq:proof theorem 5}
\begin{split}
   & \left( \hat{w}^{iT} \hat{w}^i+\epsilon \right)^\frac{p}{2}
     -\frac{p}{2} \frac{\hat{w}^{iT} \hat{w}^i+\epsilon}{ \left( w^{iT} w^i+\epsilon \right)^\frac{2-p}{2}} \leq \\
   & \left( w^{iT} w^i+\epsilon \right)^\frac{p}{2} -\frac{p}{2} \frac{w^{iT} w^i+\epsilon}{\left( w^{iT} w^i+\epsilon \right)^\frac{2-p}{2}}
\end{split}
\end{equation}

Due to $\gamma>0$, we can further get
\begin{equation}\label{eq:proof theorem 6}
\begin{split}
   & \gamma \sum_{i=1}^{d} \left( \hat{w}^{iT} \hat{w}^i+\epsilon \right)^\frac{p}{2}
     -\gamma \sum_{i=1}^{d} \frac{p}{2} \frac{\hat{w}^{iT} \hat{w}^i+\epsilon}{ \left( w^{iT} w^i+\epsilon \right)^\frac{2-p}{2}} \leq \\
   & \gamma \sum_{i=1}^{d} \left( w^{iT} w^i+\epsilon \right)^\frac{p}{2} -
   \gamma \sum_{i=1}^{d}\frac{p}{2} \frac{w^{iT} w^i+\epsilon}{\left( w^{iT} w^i+\epsilon \right)^\frac{2-p}{2}}
\end{split}
\end{equation}

For Eq.~\eqref{eq:proof theorem 4} and Eq.~\eqref{eq:proof theorem 6}, we add the left parts of the two inequalities.
For the right parts,
we perform the similar operation.
Thus, we have
\begin{equation}\label{eq:proof theorem 7}
\begin{split}
   &-\operatorname{Tr}(\hat{W}^T S_t \hat{W})+
   \gamma \sum_{i=1}^{d} \left( \hat{w}^{iT} \hat{w}^i+\epsilon \right)^\frac{p}{2} \leq \\
   &-\operatorname{Tr}(W^T S_t W)+
   \gamma \sum_{i=1}^{d} \left( w^{iT} w^i+\epsilon \right)^\frac{p}{2}
\end{split}
\end{equation}
By comparing Eq.~\eqref{eq:proof theorem 7} with problem~\eqref{eq:problem transition 3},
we can infer Theorem~\ref{theorem:2} holds.
\end{proof}

\subsection{Computational complexity analysis}

Since normalization is a prerequisite for all data mining tasks,
we don't count its computational cost into feature selection methods.
The computational complexity of Algorithm~\ref{algorithm:SPCAFS} can be decomposed into the following aspects:
\begin{itemize}
  \item We need $O(d^2n)$ to initialize $S_t$ and $G$,
        based on the normalized data.
  \item For one iteration, we need $O(d^3)$ to update $W$ by performing eigen-decomposition of $(-S_t+\gamma G)$.
    \item For one iteration, we need $O(dm)$ to update $G$ according to Eq.~\eqref{eq:G definition}.
    \item We need $O(dm)$ to calculate $\|w^i\|_2$ ($i = 1, 2, \dots, d$) and $O(d{\rm{log}}d)$ to complete the sorting.
\end{itemize}

Thus,
the overall computational complexity is $O(d^2n+d^3t)$,
where $t$ is the number of iterations
of Algorithm~\ref{algorithm:SPCAFS}.
Note that, Algorithm~\ref{algorithm:SPCAFS} is efficient
and always converges within 30 iterations in our experiments.
We further compare its computational complexity in one iteration
with that of other competing methods.
In Table~\ref{tab: Comparison of computational complexity},
apart from some introduced notations,
$c$ is the number of clusters,
$p$ is the number of neighbours in graph construction,
$h$ is the number of selected features.
We can conclude that:
\begin{itemize}
  \item For embedded methods,
   the computational complexity of sparse regression usually contains $d^3$, which is produced by inverse operation or eigen-decomposition.
  \item SPCAFS does not require the construction of a similarity matrix by KNN, which will need at least $O(dn^2)$.
  \item Of all the methods, only the computational complexity of SPCAFS is linear to $n$.
  It indicates that the computational complexity of SPCAFS will not
expand rapidly, as the number of samples increases.

%
%


\end{itemize}

\begin{table}[htbp]
\caption{Comparison of computational complexity.}\label{tab: Comparison of computational complexity}
\centerline{
\scalebox{1}{
\begin{tabular}{p{2.5cm}<{\centering}p{5cm}<{\centering}}
\hline
      \textbf{Methods}  & \textbf{Computational complexity} \\
    \hline
     LapScore &   $O(dn^2+d{\rm{log_2}}d)$ \\
     MCFS &  $O(d^3+n^2m+d^2n)$ \\
     UDFS &  $O(d^3+n^2c)$ \\
     EUFS &  $O(dn^2+nc^2)$ \\
     DGUFS & $O(n^3+dn)$  \\
     SOGFS &   $O(d^3+n^2p+n^2c)$ \\
     RNE &   $O(dn^2+d^2n+d^2h)$ \\
     SPCAFS &  $O(d^2n+d^3)$ \\
    \hline
\end{tabular}}
}
\end{table}

\section{Experimental evaluation}\label{sec:experiments}
\subsection{Experimental setup }

\subsubsection{Experimental environment and data sets }
Hardware is a workstation with 3.8 GHz CPU and 16 GB RAM.
The experimental environment is Windows 64-bit Operating System,
running Matlab R2018a.
Our experiments were executed on 6 publicly available data sets,
including four image data sets PalmData25, Imm40, PIE and AR,
two bioinformatics data sets SRBCTML and LEUML.
More information of the data sets is shown in Table~\ref{tab:data sets}.

\begin{table}[htbp]
\caption{The details of the experimental data sets.}\label{tab:data sets}
\centerline{
\scalebox{1}{
\begin{tabular}{p{2cm}p{1.6cm}<{\centering}p{1.6cm}<{\centering}p{1.6cm}<{\centering}}
\hline
     \textbf{Data sets} &  \textbf{\# of Feature}
     &	\textbf{\# of Instance} &	\textbf{\# of Class} \\
    \hline
     PalmData25 &256 &2000  &100 \\
     Imm40 &1024 &240  &40 \\
     PIE &1024 &1166  &53 \\
     AR &2200 &2600  &100 \\
     SRBCTML &2308 &83  &4 \\
     LEUML &3571 &72  &2 \\
    \hline
\end{tabular}}
}
\end{table}

\subsubsection{Comparision methods and parameters setting }

To verify the effectiveness of the proposed method,
we compared it with several state-of-the-art methods in the field of unsupervised feature selection,
such as LapScore, MCFS, UDFS, EUFS,
DGUFS, SOGFS and RNE.
To get the baseline for analysis, all features are selected as a special case of feature selection.

To ensure that the experiments are as fair as possible,
we adopt the same strategy to set parameters for all the unsupervised feature selection methods.
For LapScore, MCFS, UDFS, EUFS, DGUFS, SOGFS and RNE,
we set the neighborhood size to be 5.
For EUFS, SOGFS and SPCAFS,
the reduced dimension is fixed as $m=c-1$.
We tune all the parameters by grid search strategy from  \{$10^{-6},10^{-4},10^{-2},10^{0},10^{2},10^{4},10^{6}$ \}.
Without loss of generality,
we set $p=1$ in $\ell_{2,p}$-norm regularization of SPCAFS.

\subsubsection{Evaluation metrics }

To verify the validity of SPCAFS,
we execute $K$-means clustering by inputting the results of different unsupervised  feature selection methods.
Clustering accuracy (ACC) and Normalized Mutual Information (NMI) are utilized to evaluate the effectiveness of feature selection indirectly.
ACC is defined as~\cite{Zhu2019accNMI}
\begin{equation}\label{eq:Acc}
\begin{split}
    ACC =  \frac{\sum_{i=1}^n \delta(c_i,map(l_i))}{n}
\end{split}
\end{equation}
where $n$ is the number of data points,
$c_i$ is the given cluster label,
$l_i$ is the obtained cluster label,
map(·) is the permutation mapping function that maps each obtained cluster label $l_i$ to the equivalent label from the data set.
The best mapping can be found by using Kuhn–Munkres algorithm~\cite{Strehl2002Cluster}.
$\delta$ is a function defined as
\begin{equation}\label{eq:delta function}
\begin{split}
    & \delta(a,b)=  \left\{ {\begin{array}{l}
  1,  \quad if \ a=b \\
  0,  \quad otherwise \\
 \end{array}} \right.  \\
\end{split}
\end{equation}

NMI is defined as~\cite{Zhu2019accNMI}
\begin{equation}\label{eq:NMI}
\begin{split}
    NMI =  \frac{MI(C, C'))}{max(H(C), H(C'))}
\end{split}
\end{equation}
where $C$ is the set of clusters obtained from the ground truth
and $C'$ is the set of clusters computed by a clustering algorithm.
$MI(C, C)$ is the mutual information metric.
$H(C)$, $H(C')$ are the entropies of $C$ and $C'$ respectively.

For each unsupervised  feature selection method,
the best result of $K$-means clustering with the optimal parameters is recorded.
Since the result of $K$-means clustering depends on initialization,
we repeated $K$-means clustering 20 times for all the methods,
and report their average results.

\subsection{Clustering results analysis with selected features}

Without loss of generality,
we set the number of selected features as \{$10,20,30,40,50,60,70,80,90,100$ \} for each data set.
The experimental results of ACC, NMI are illustrated in Fig.~\ref{fig: Acc results},  Fig.~\ref{fig: NMI results} respectively.
We can get the following conclusions:

\begin{itemize}
  \item[1.] Feature selection of SPCAFS is effective.
Compared with the baseline (AllFea),
both ACC and NMI of SPCAFS have been significantly improved
on almost all these data sets.
On PalmData25 data set,
although ACC and NMI of SPCAFS are lower than the baseline, they are still higher than those of other methods.
As the number of selected features increases,
both ACC and NMI of SPCAFS gradually approach the baseline.

  \item[2.] In general, with the increase of the number of selected features, the curves of clustering results show a trend of rising first and then falling.
Because data sets from practical applications usually contain many redundancy features and a few discriminative features.
As the number of selected features increases,
some redundancy features are selected,
decreasing the clustering performance of feature selection methods.

 \item[3.] The performance of SPCAFS exceeds other competing methods on all these data sets.
In particular on Imm40 data set,
compared with the second best method MCFS,
SPCAFS has 5 percent improvement of ACC,
3 percent improvement of NMI on average.

\end{itemize}

\begin{figure*}[!htbp]
\centering
\subfloat[Imm40]{\includegraphics[width=0.32\textwidth ]{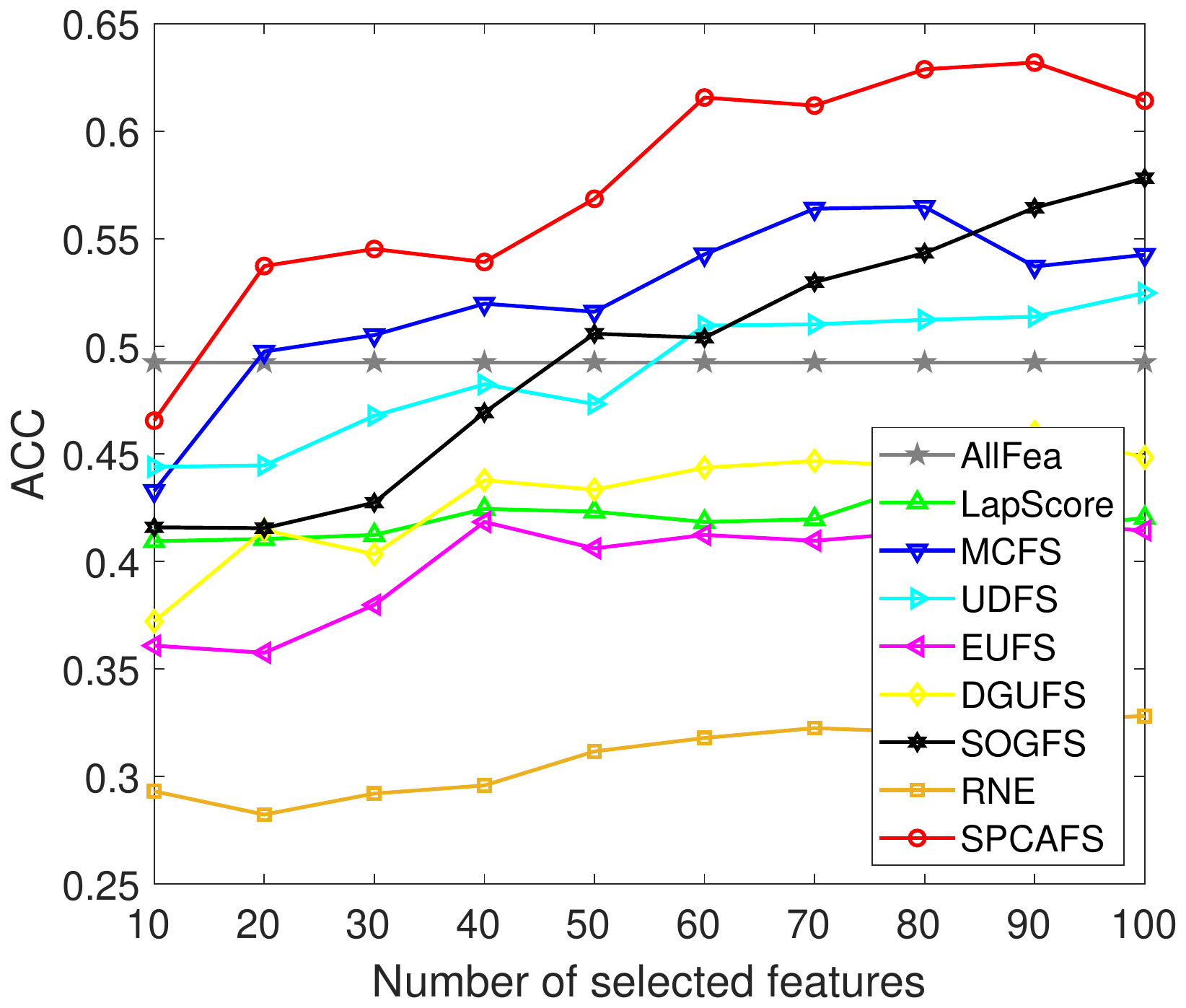}}
~~\subfloat[PIE]{\includegraphics[width=0.32\textwidth ]{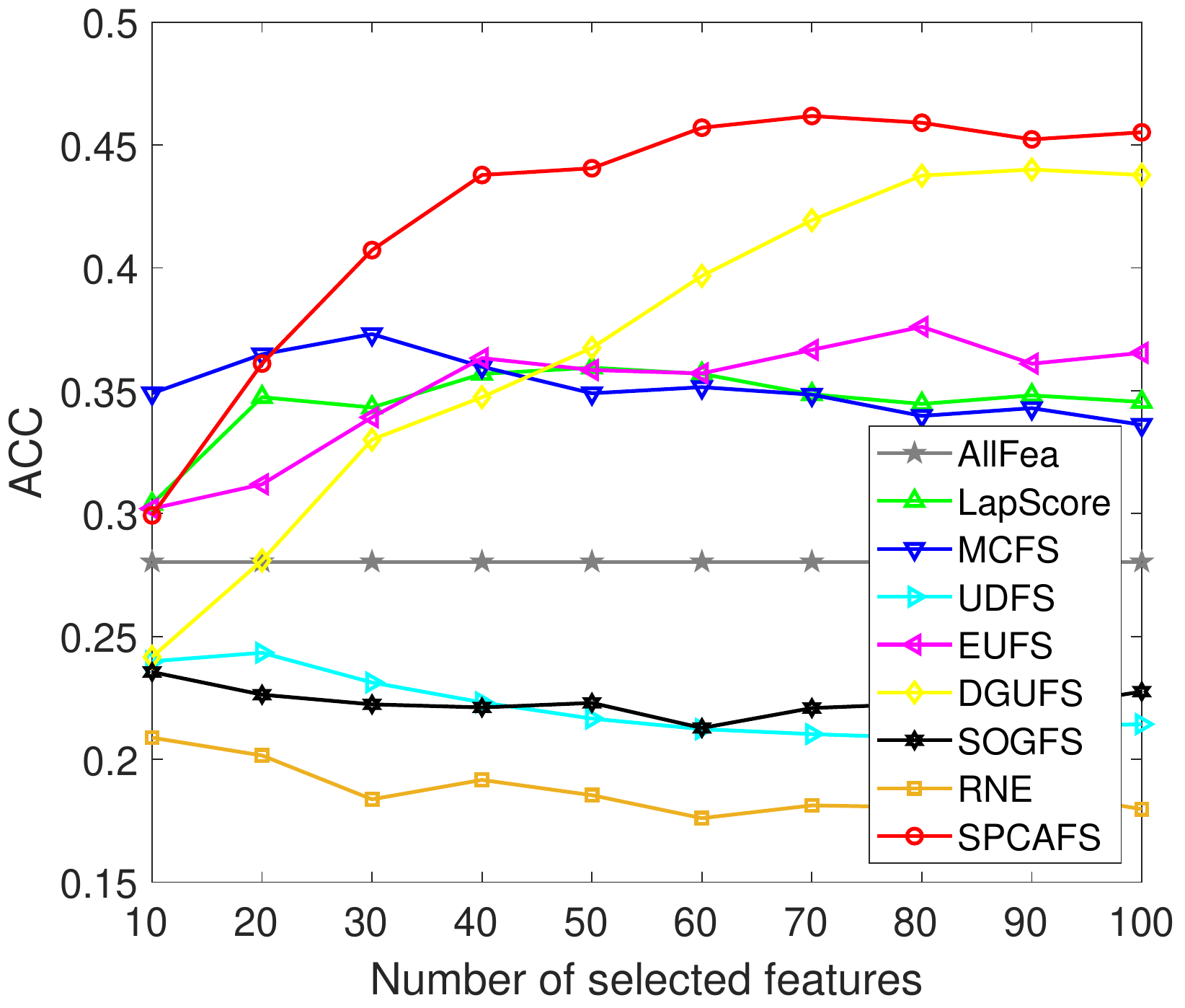}}
~~\subfloat[SRBCTML]{\includegraphics[width=0.32\textwidth ]{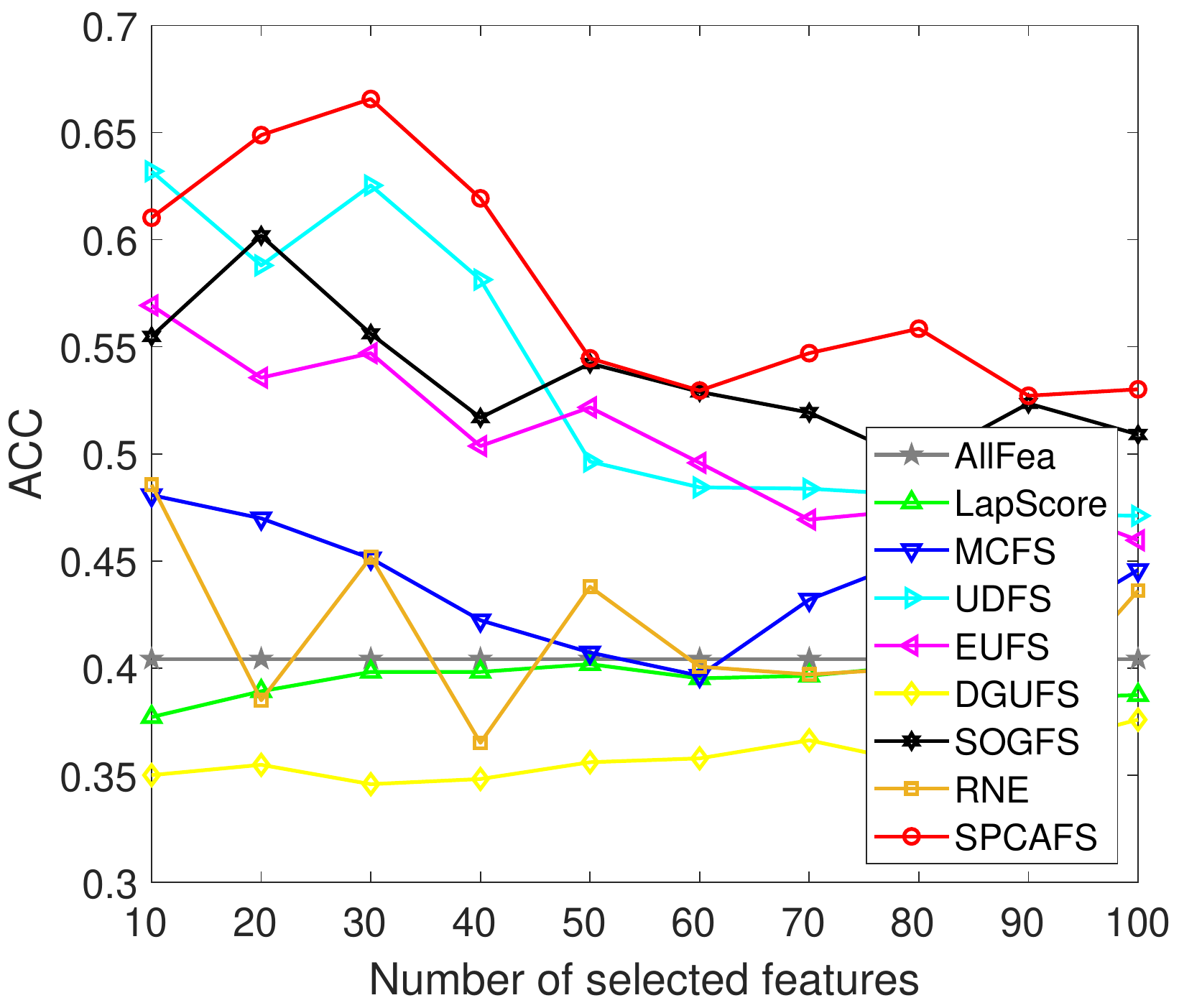}}\\
\subfloat[AR]{\includegraphics[width=0.32\textwidth ]{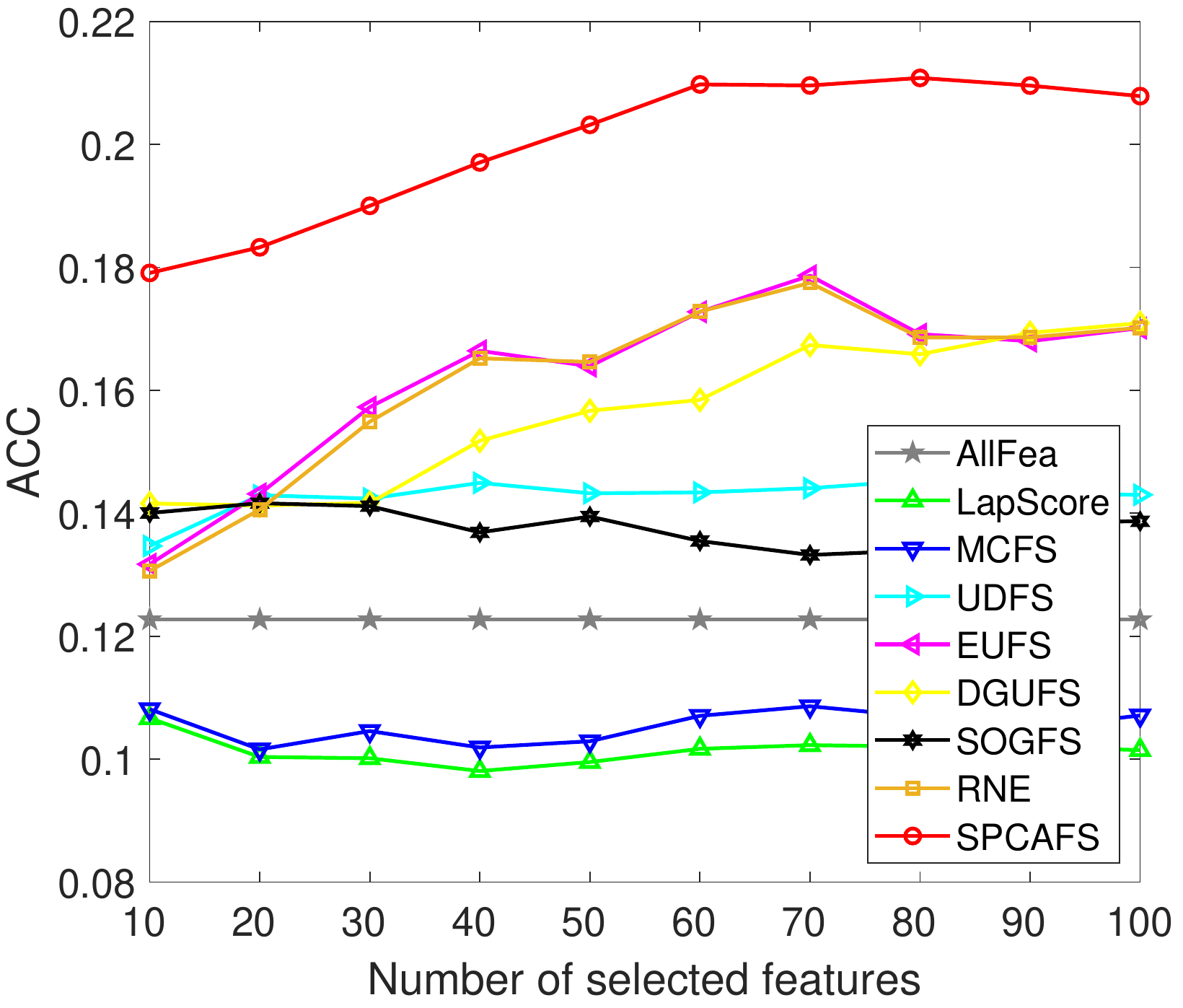}}
~~\subfloat[LEUML]{\includegraphics[width=0.32\textwidth ]{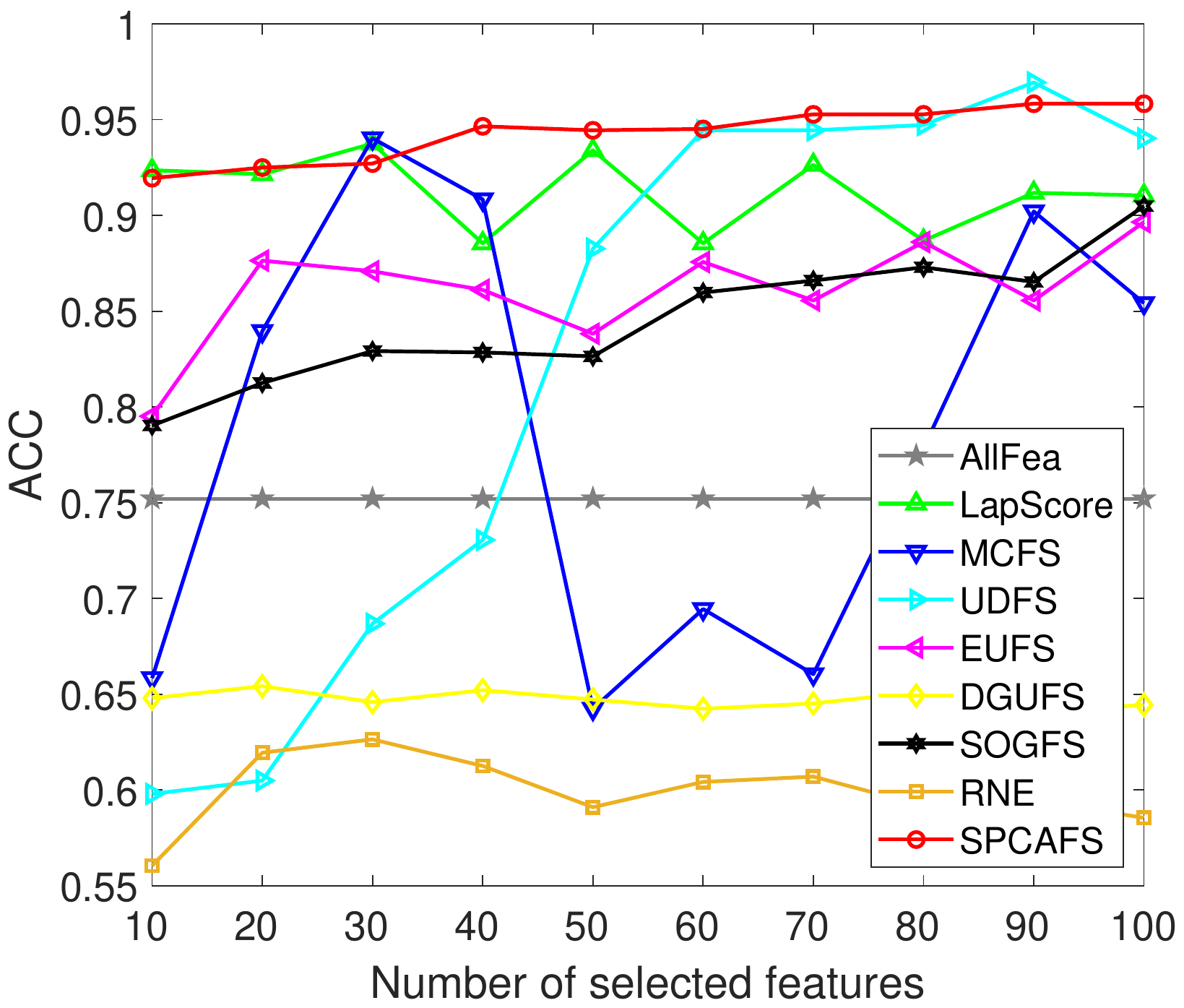}}
~~\subfloat[PalmData25]{\includegraphics[width=0.32\textwidth ]{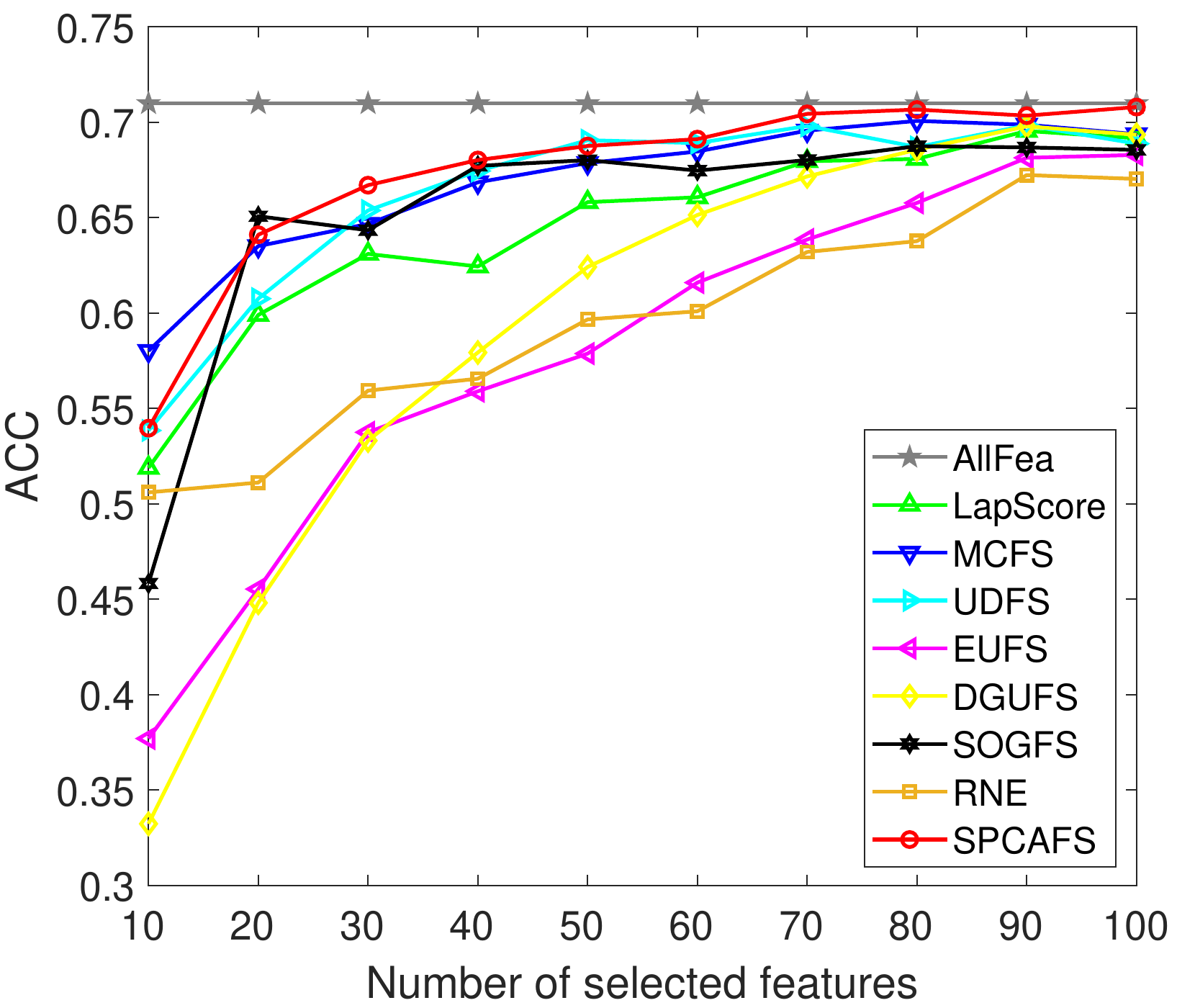}}
\caption{Clustering results (ACC) of different unsupervised feature selection methods.}
\label{fig: Acc results}
\end{figure*}

\begin{figure*}[!htbp]
\centering
\subfloat[Imm40]{\includegraphics[width=0.32\textwidth ]{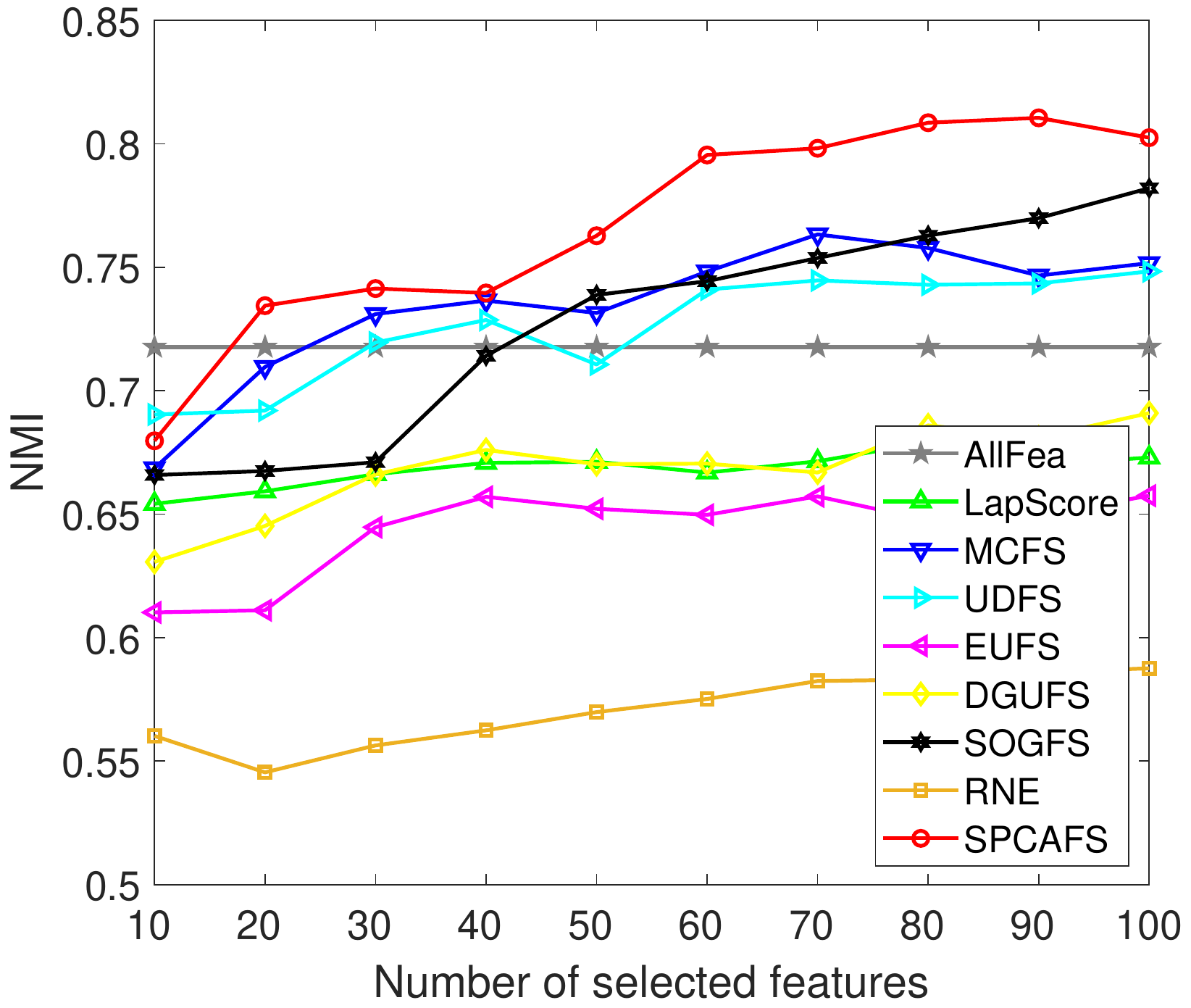}}
~~\subfloat[PIE]{\includegraphics[width=0.32\textwidth ]{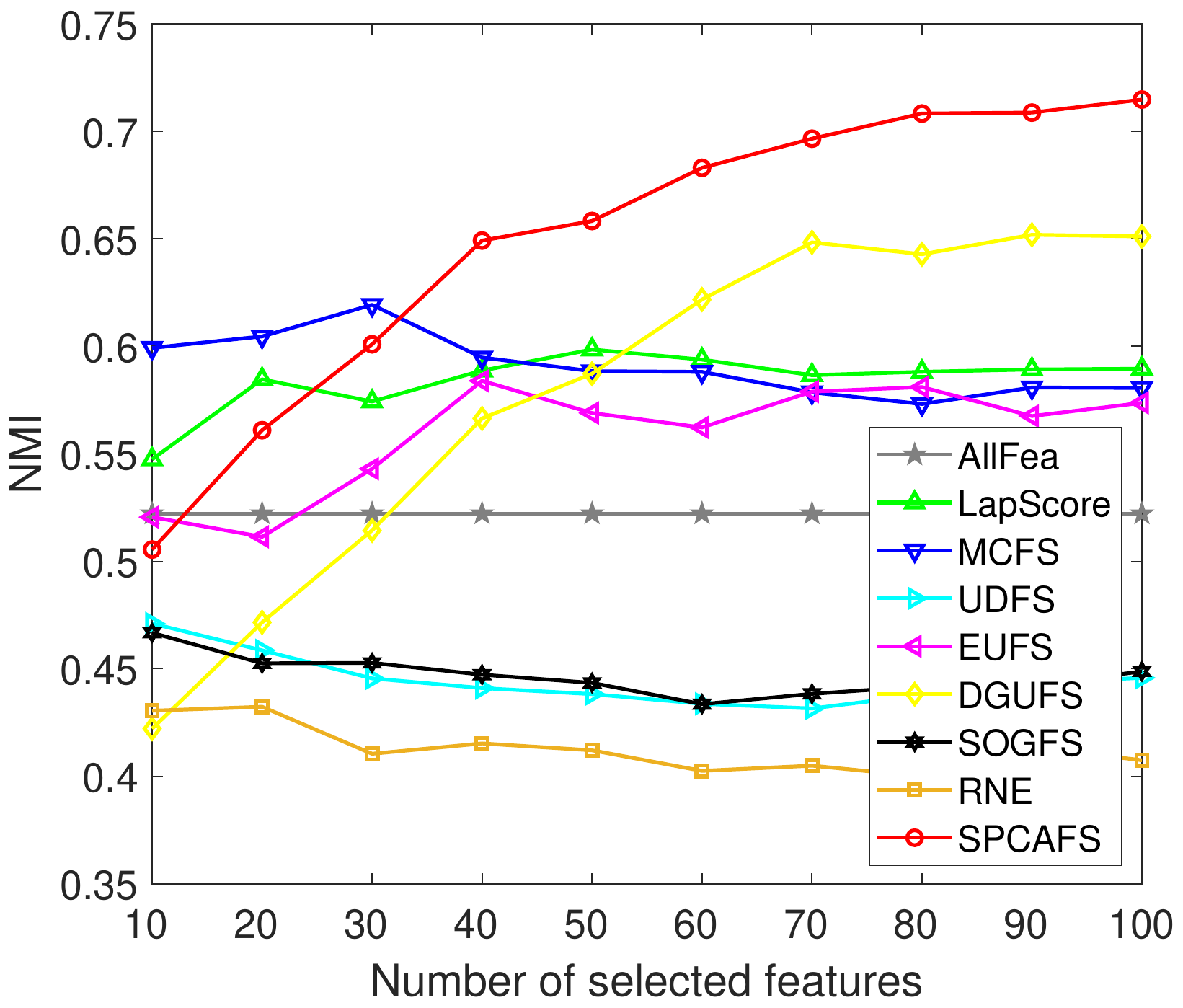}}
~~\subfloat[SRBCTML]{\includegraphics[width=0.32\textwidth ]{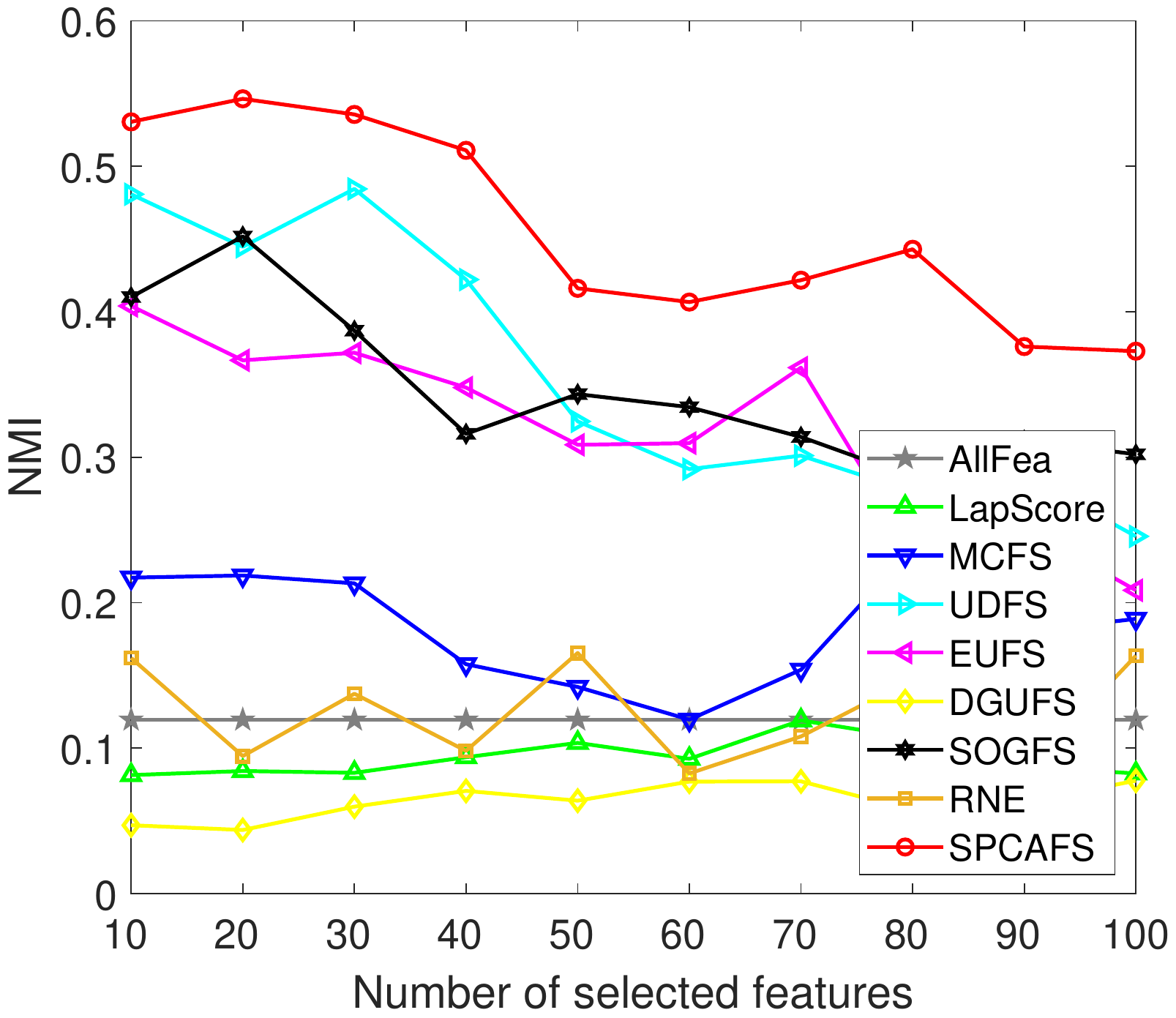}}\\
\subfloat[AR]{\includegraphics[width=0.32\textwidth ]{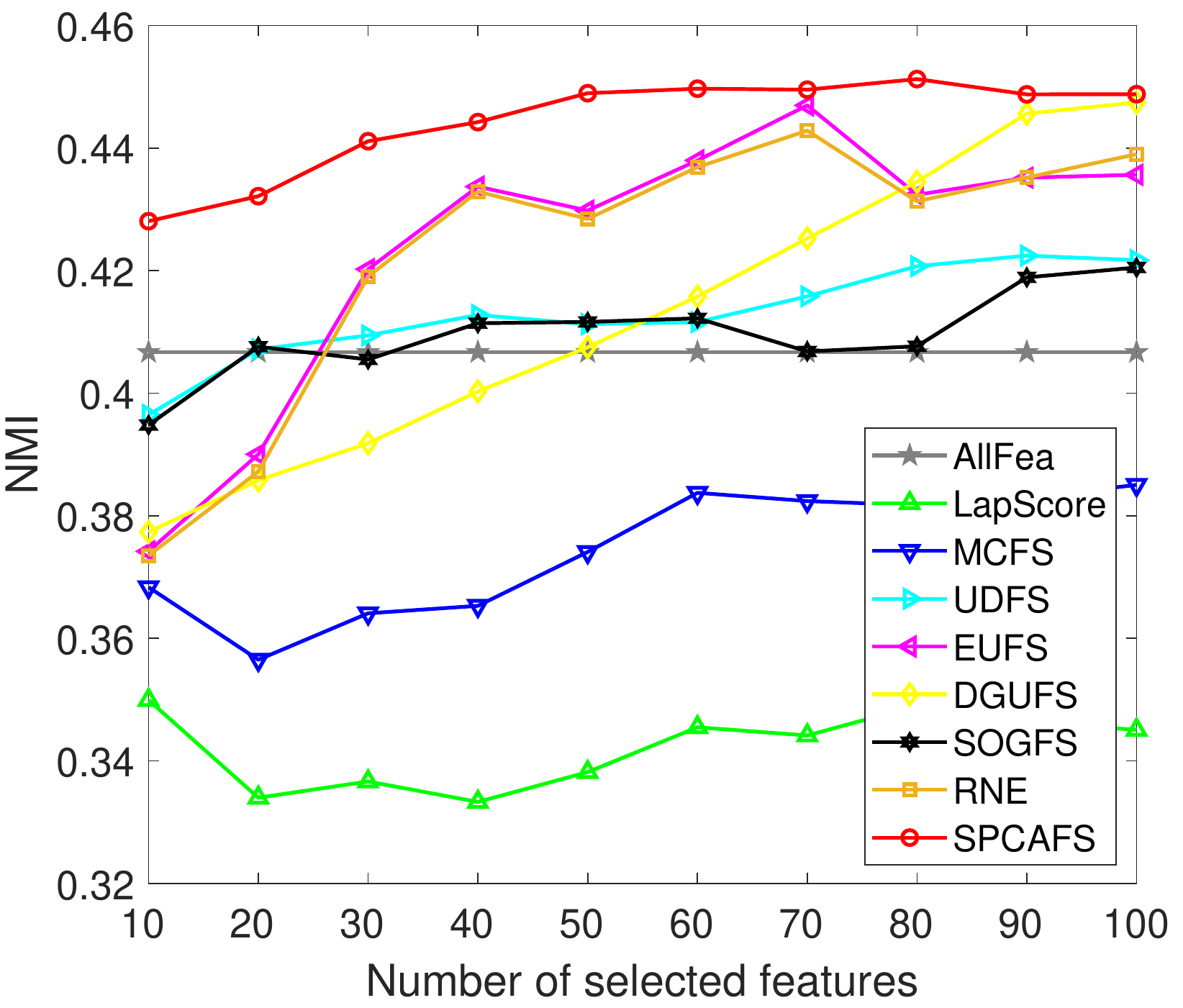}}
~~\subfloat[LEUML]{\includegraphics[width=0.32\textwidth ]{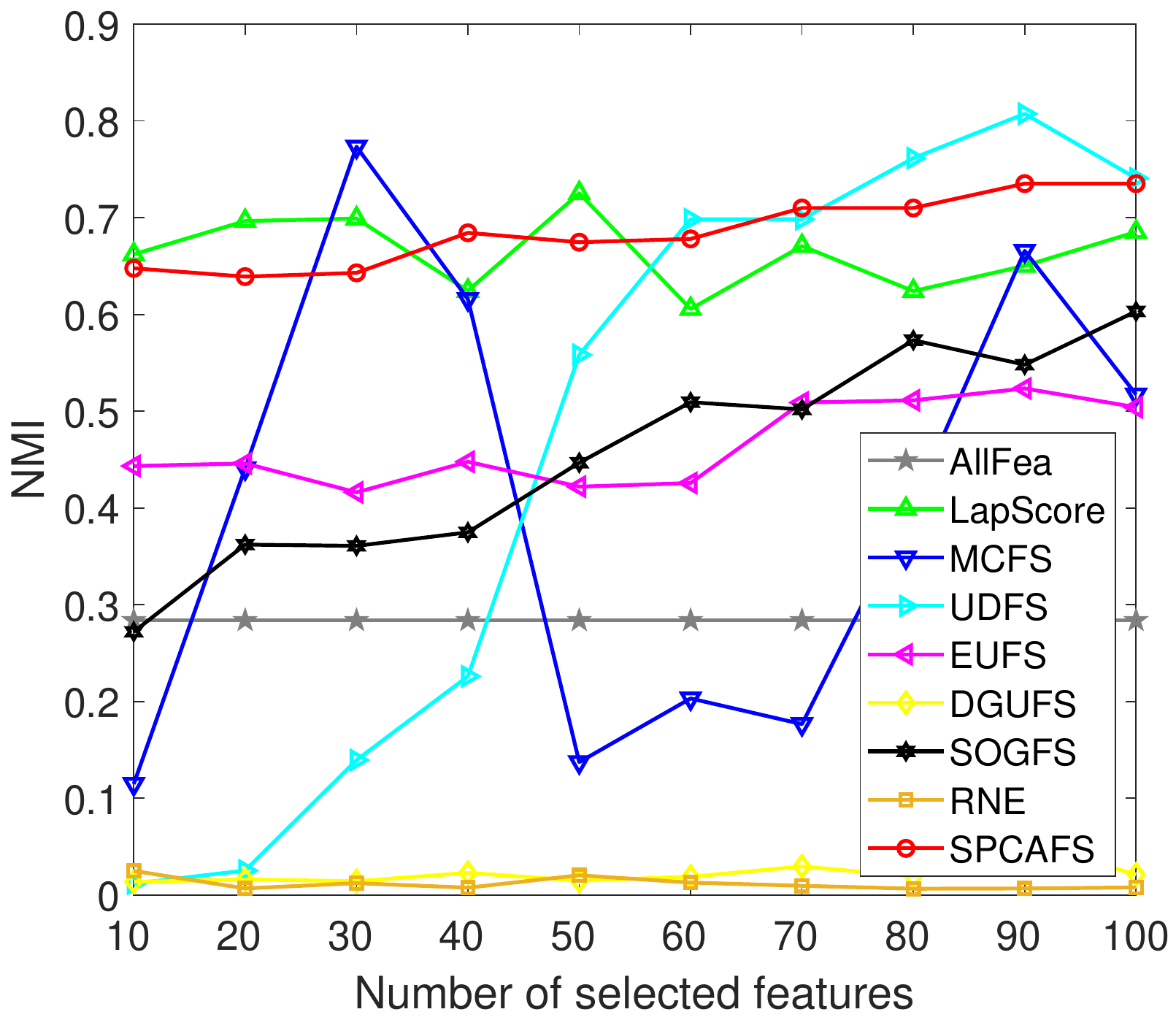}}
~~\subfloat[PalmData25]{\includegraphics[width=0.32\textwidth ]{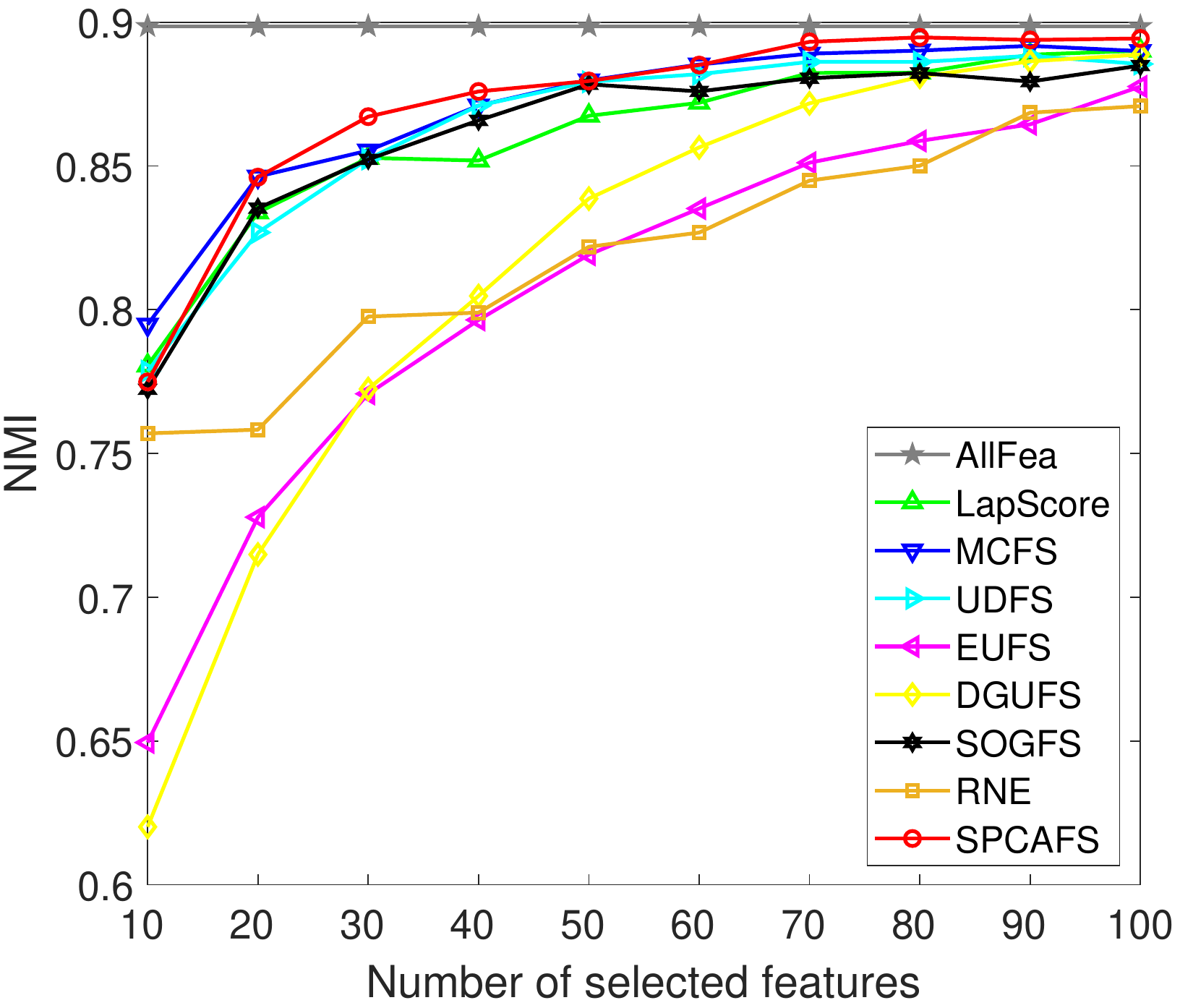}}
\caption{Clustering results (NMI) of different unsupervised feature selection methods.}
\label{fig: NMI results}
\end{figure*}


\subsection{  Convergence study }

In Section~\ref{sec:Convergence analysis},
we have proven the convergence of Algorithm~\ref{algorithm:SPCAFS}.
We further study the speed of its convergence by experiments.
The convergence curves of the objective value are demostrated in Fig.~\ref{fig: Convergence curve}.
Due to space limitation,
we only show the results on four data sets.
We can see that the speed of convergence of Algorithm~\ref{algorithm:SPCAFS} is very fast,
which ensures the efficiency of SPCAFS.

\begin{figure}[!htbp]
\centering
\subfloat[PalmData25]{\includegraphics[width=0.24\textwidth ]{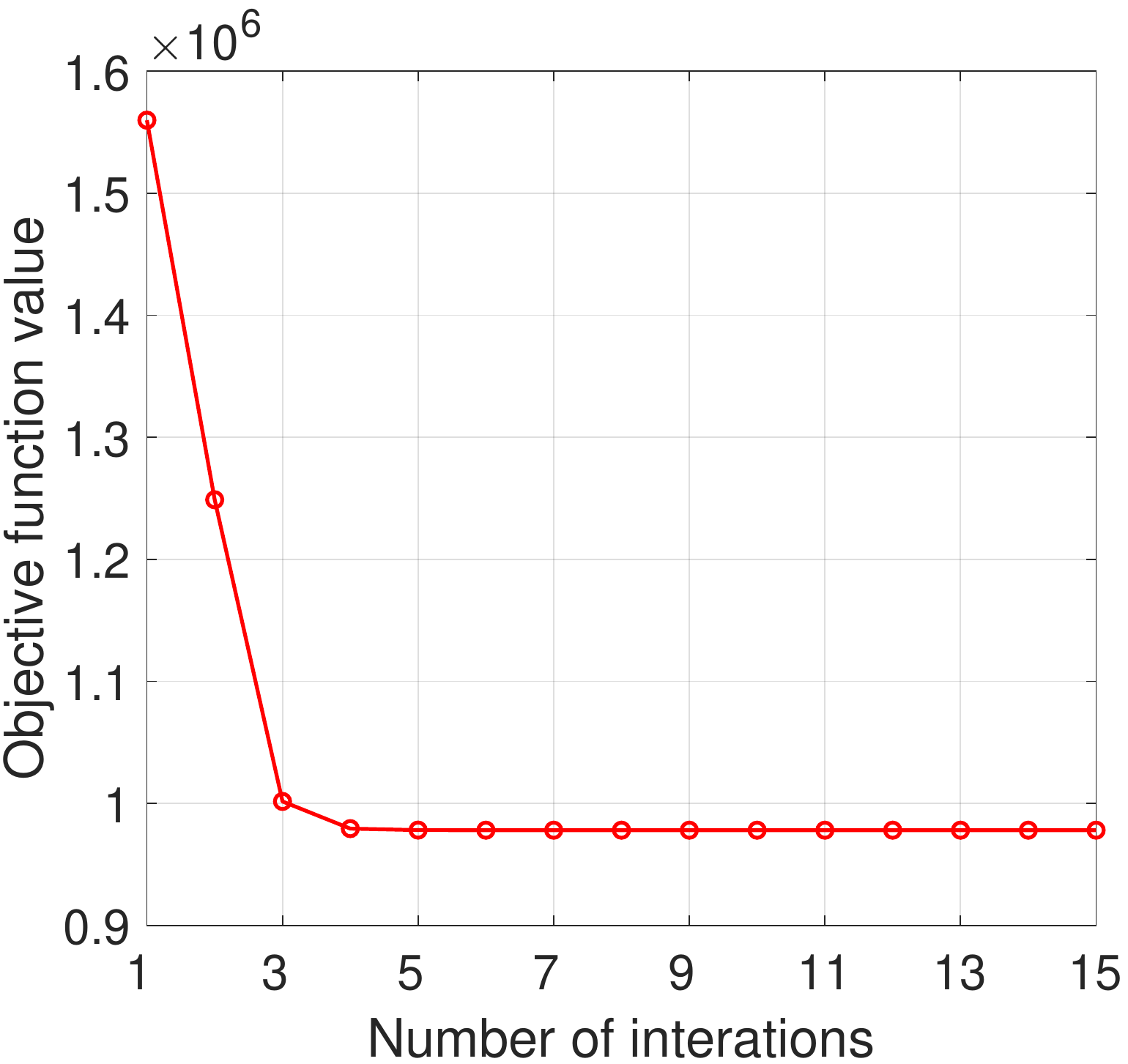}}
~\subfloat[Imm40]{\includegraphics[width=0.24\textwidth ]{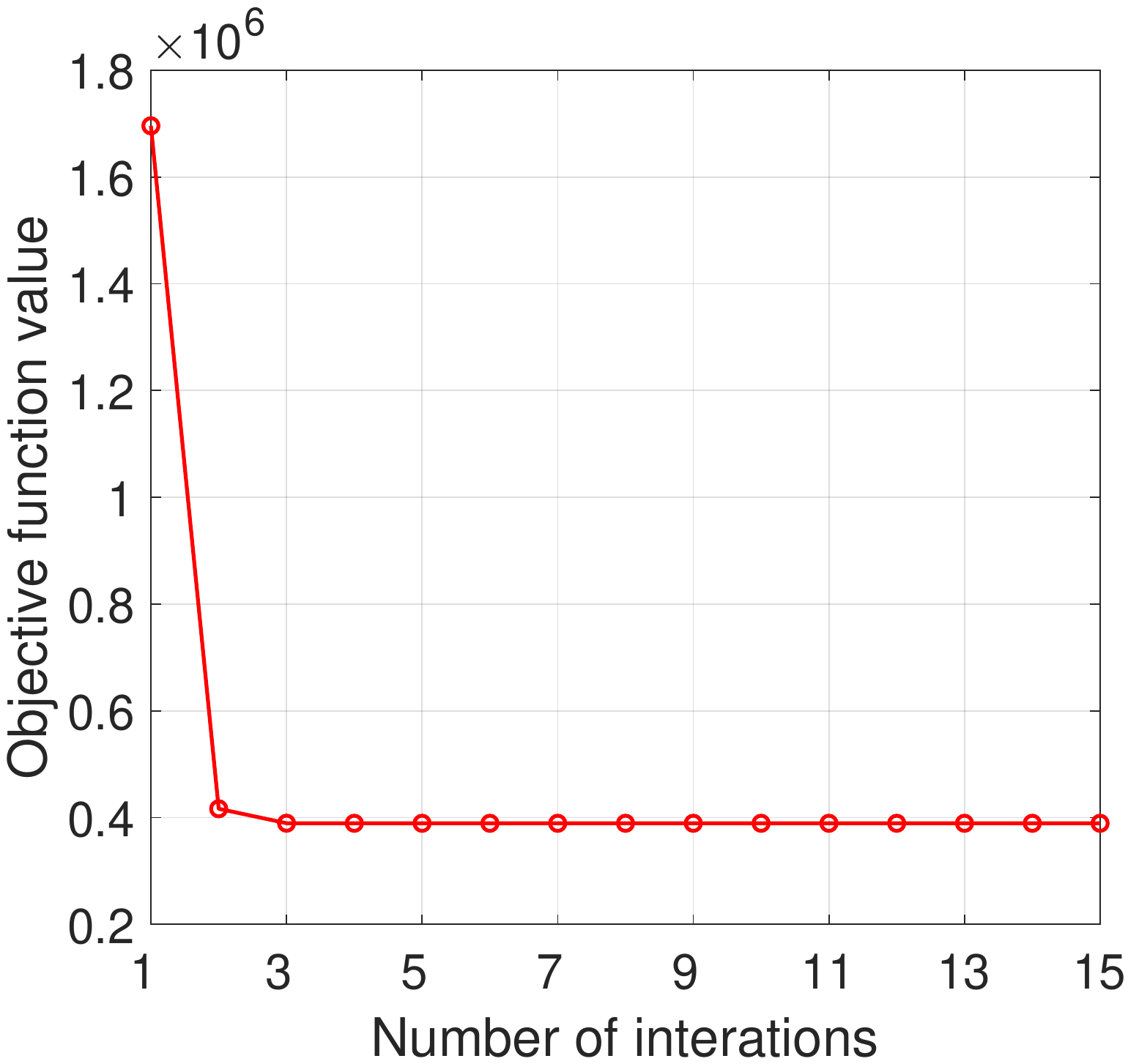}}\\
\subfloat[SRBCTML]{\includegraphics[width=0.24\textwidth ]{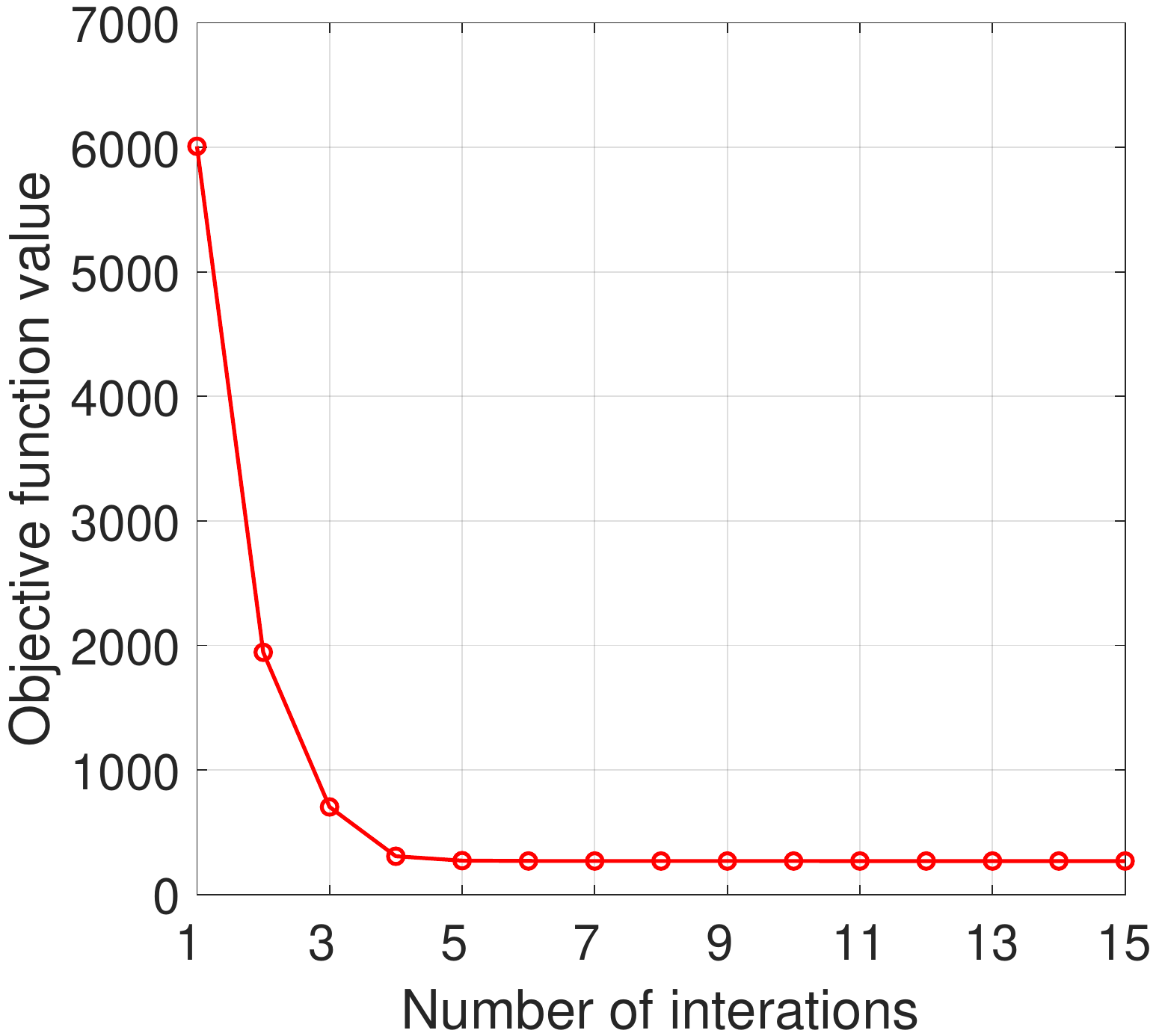}}
~\subfloat[LEUML]{\includegraphics[width=0.24\textwidth ]{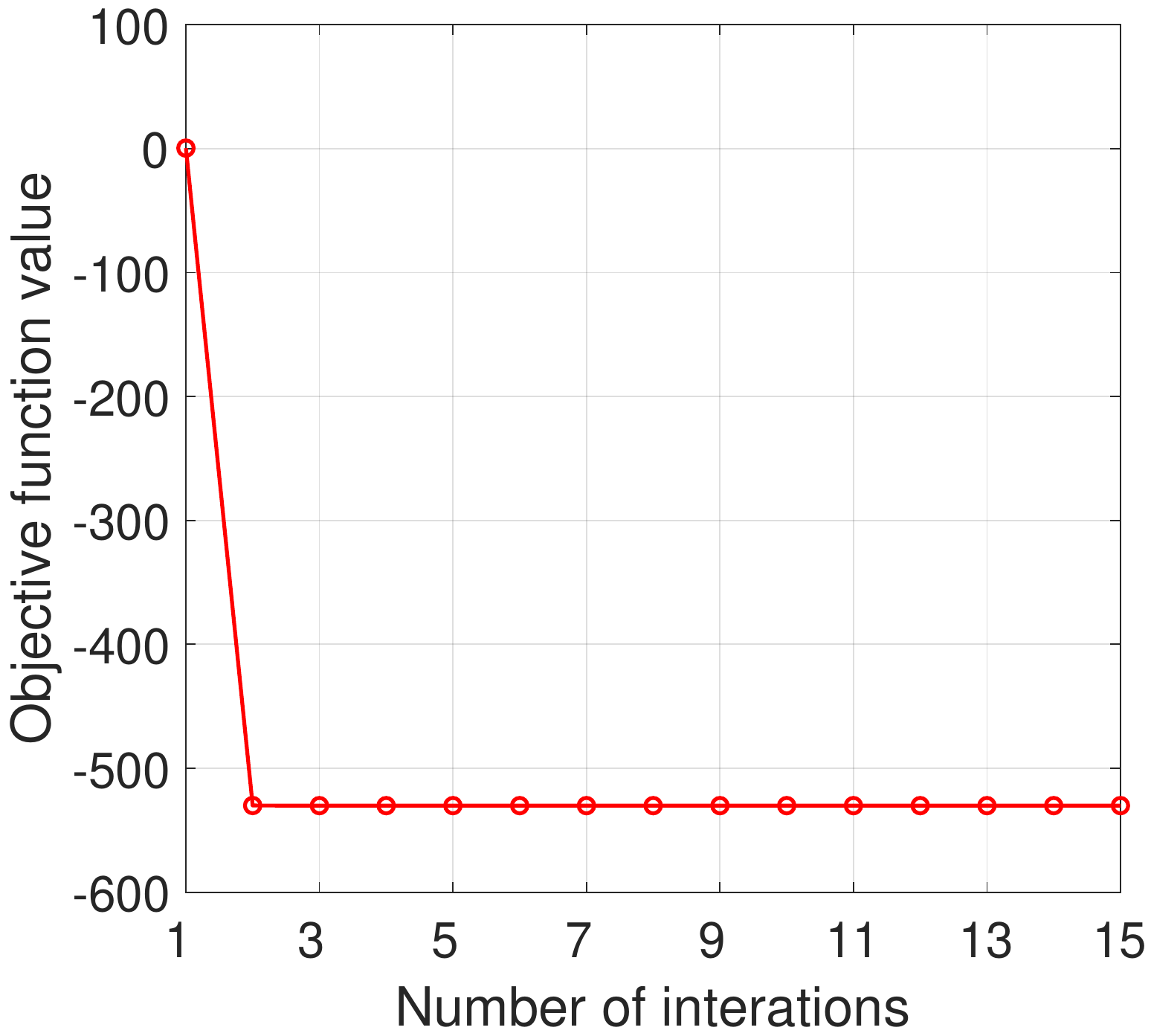}}
\caption{Convergence curve of SPCAFS on PalmData25, Imm40, SRBCTML and LEUML data sets.}
\label{fig: Convergence curve}
\end{figure}

\subsection{ Parameter sensitivity analysis}

\subsubsection{ Sensitivity analysis for the parameters $\gamma$, $m$ }

We further investigate the impact of parameters $\gamma$, $m$ on SPCAFS.
The experimental results on all these data sets are similar.
Due to space limitation,
we only present the experimental results on Imm40 data set.
We first adjust $\gamma$ by fixing $m=c-1$.
There are some small fluctuations in clustering performance under different $\gamma$,
as illustrated in Fig.~\ref{fig: Parameter sensitivity}(a)-(b).
Because $\gamma$ is used to control the row sparsity of projection matrix $W$,
its variation will affect the result of feature selection.
Then, we adjust $m$ by fixing $\gamma=10^4$.
When $m$ changes from 10 to 70,
the clustering performance does not change significantly,
as illustrated in Fig.~\ref{fig: Parameter sensitivity}(c)-(d).
The results indicate that SPCAFS is not sensitive to parameters $\gamma$ and $m$ with wide range,
and it can be used in practical applications.

\subsubsection{The effect of $p$ in $\ell_{2,p}$-Norm regularization}

In the above experiments, we set $p=1$
in $\ell_{2,p}$-norm regularization by default.
In this section,
we discuss the effect of $p$ on the results of feature selection for SPCAFS.
Without loss of generality,
we only show the experimental results on Imm40, SRBCTML data sets,
as illustrated in Fig.~\ref{fig: p analysis}.

\begin{figure}[!htbp]
\centering
\subfloat[ACC (m=39)]{\includegraphics[width=0.24\textwidth ]{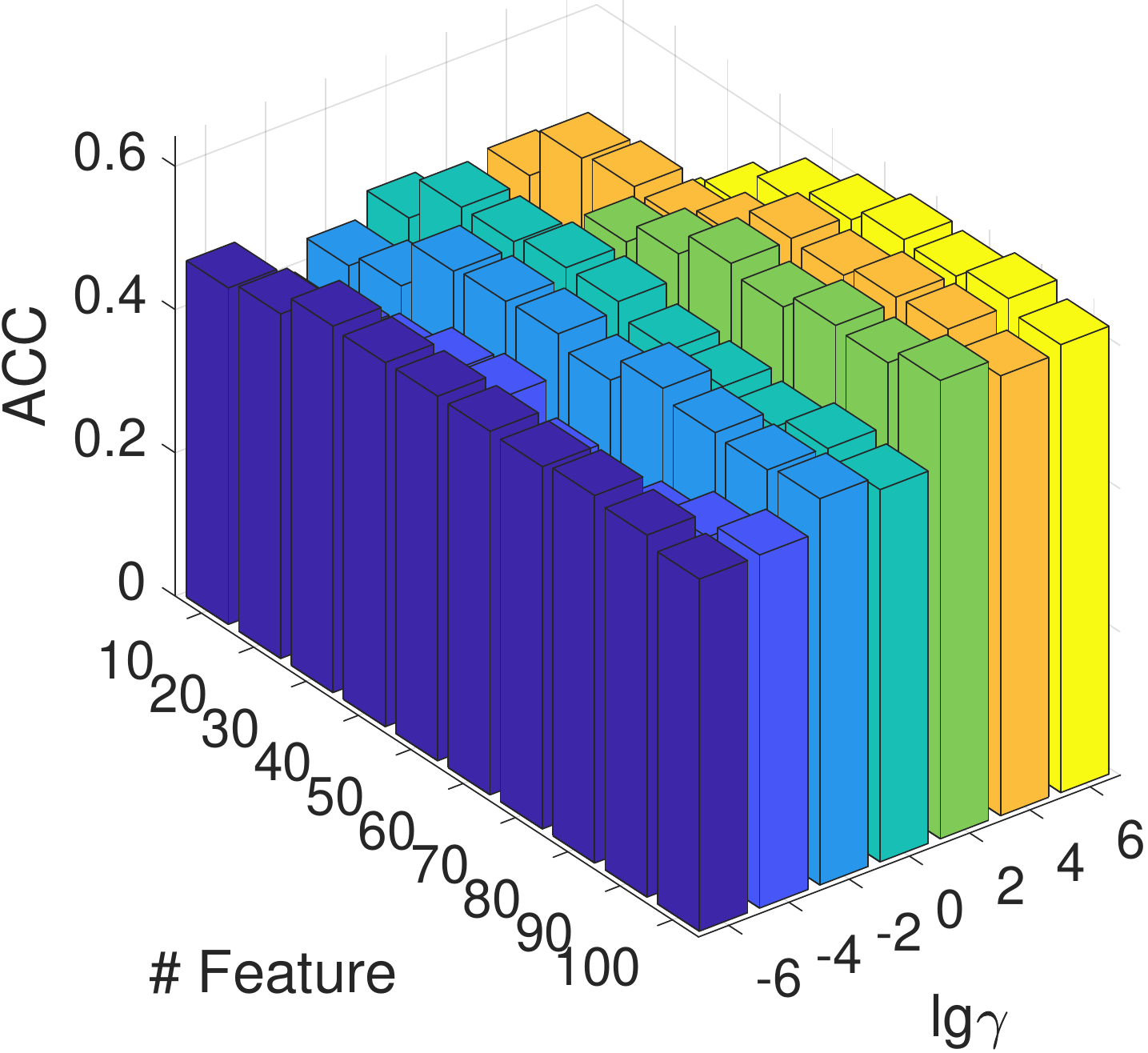}}
~\subfloat[NMI (m=39)]{\includegraphics[width=0.24\textwidth ]{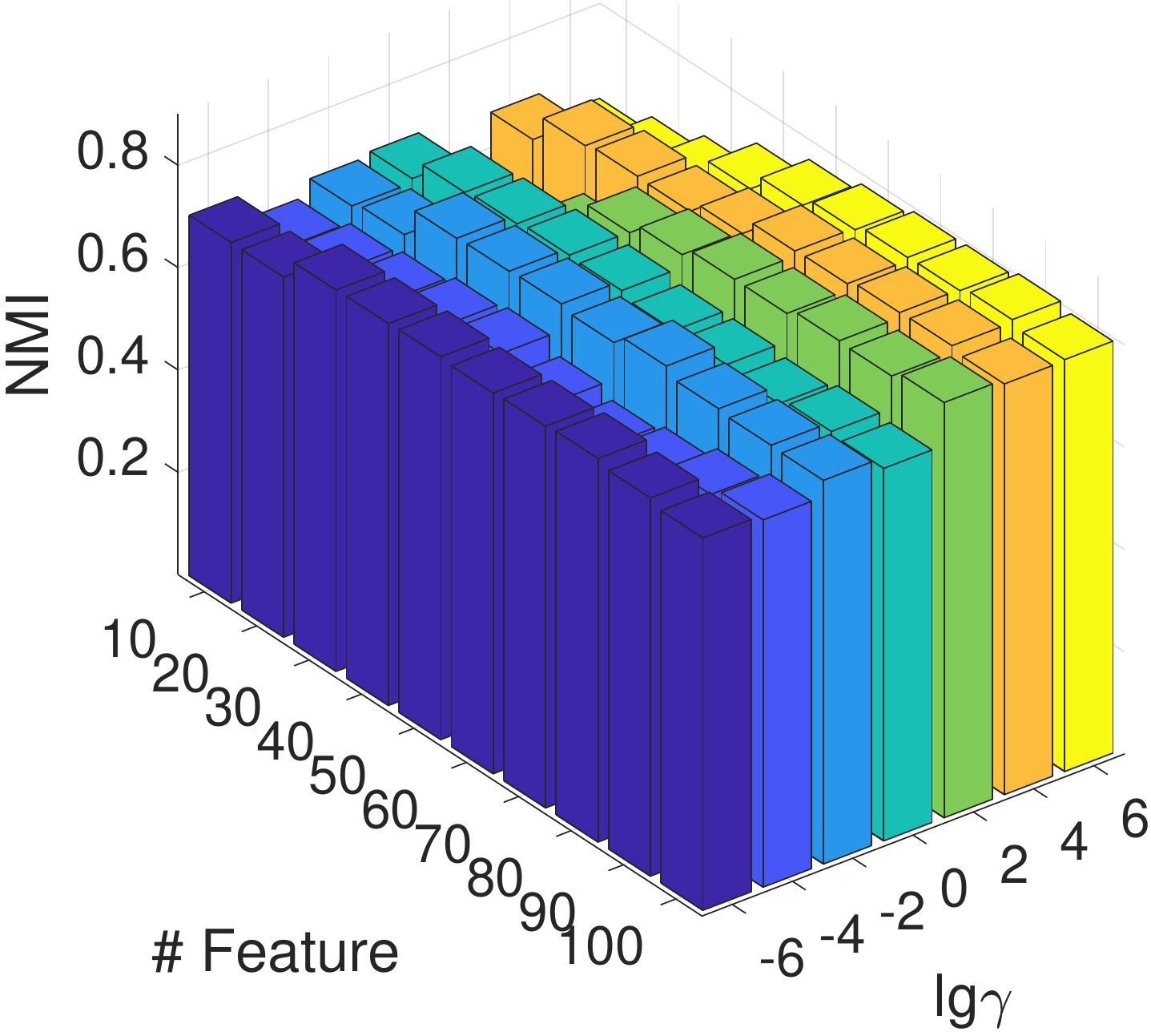}}\\
\subfloat[ACC ($\gamma=10^4$)]{\includegraphics[width=0.24\textwidth ]{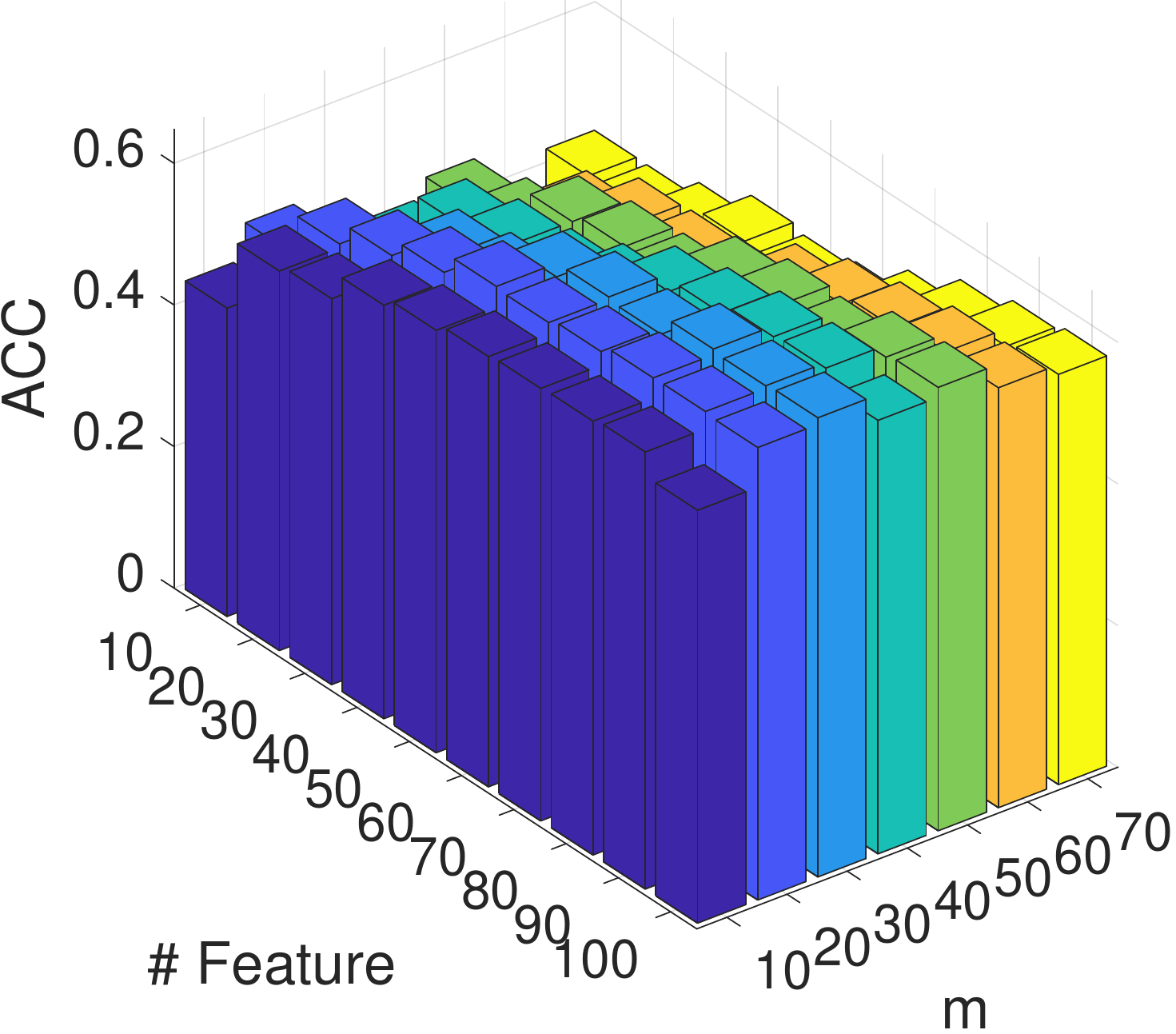}}
~\subfloat[NMI ($\gamma=10^4$)]{\includegraphics[width=0.24\textwidth ]{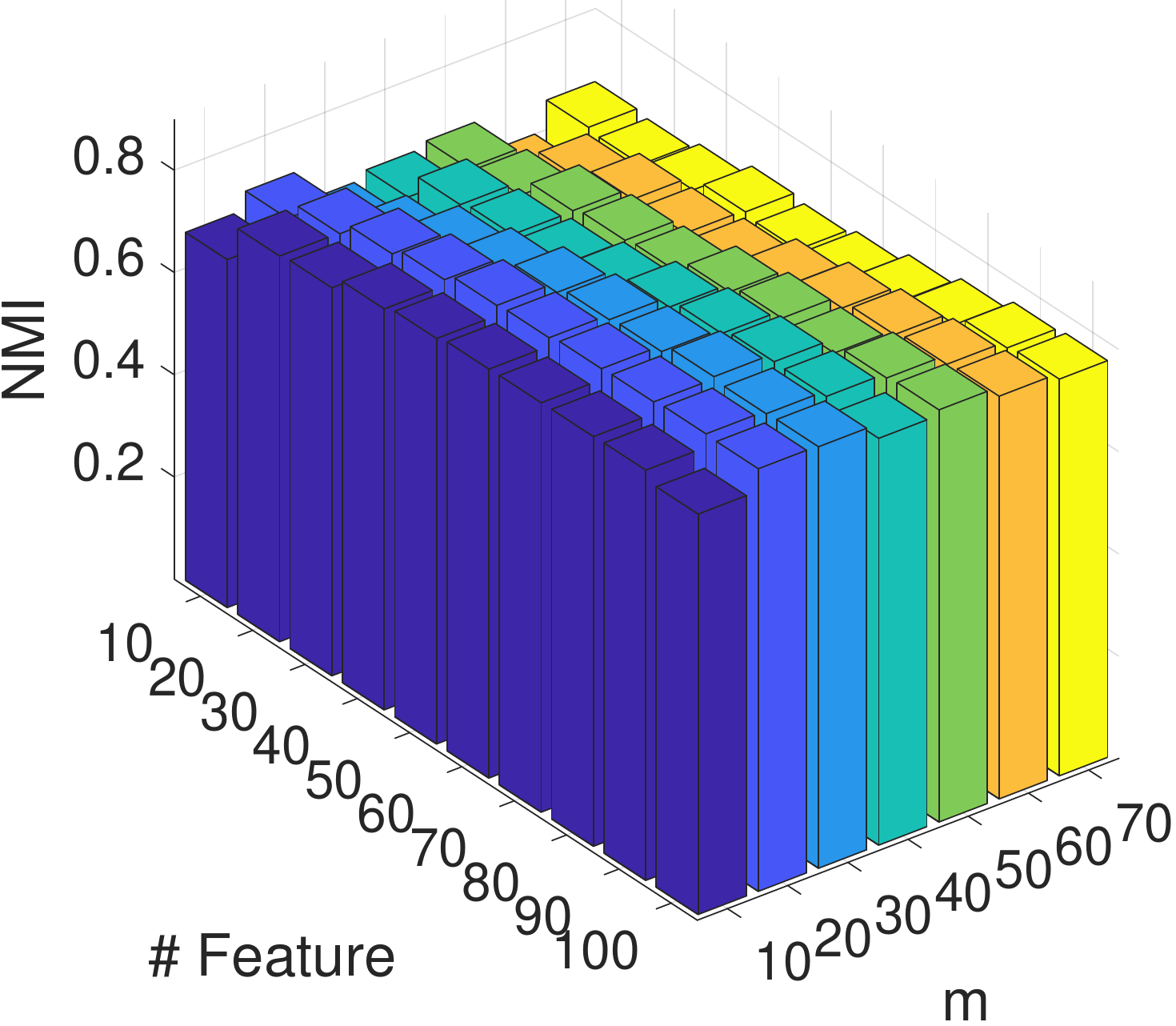}}
\caption{ACC and NMI of SPCAFS under different $\gamma$, $m$ on Imm40 data set.}
\label{fig: Parameter sensitivity}
\end{figure}

\begin{figure}[!htbp]
\centering
\subfloat[Imm40 (ACC)]{\includegraphics[width=0.25\textwidth ]{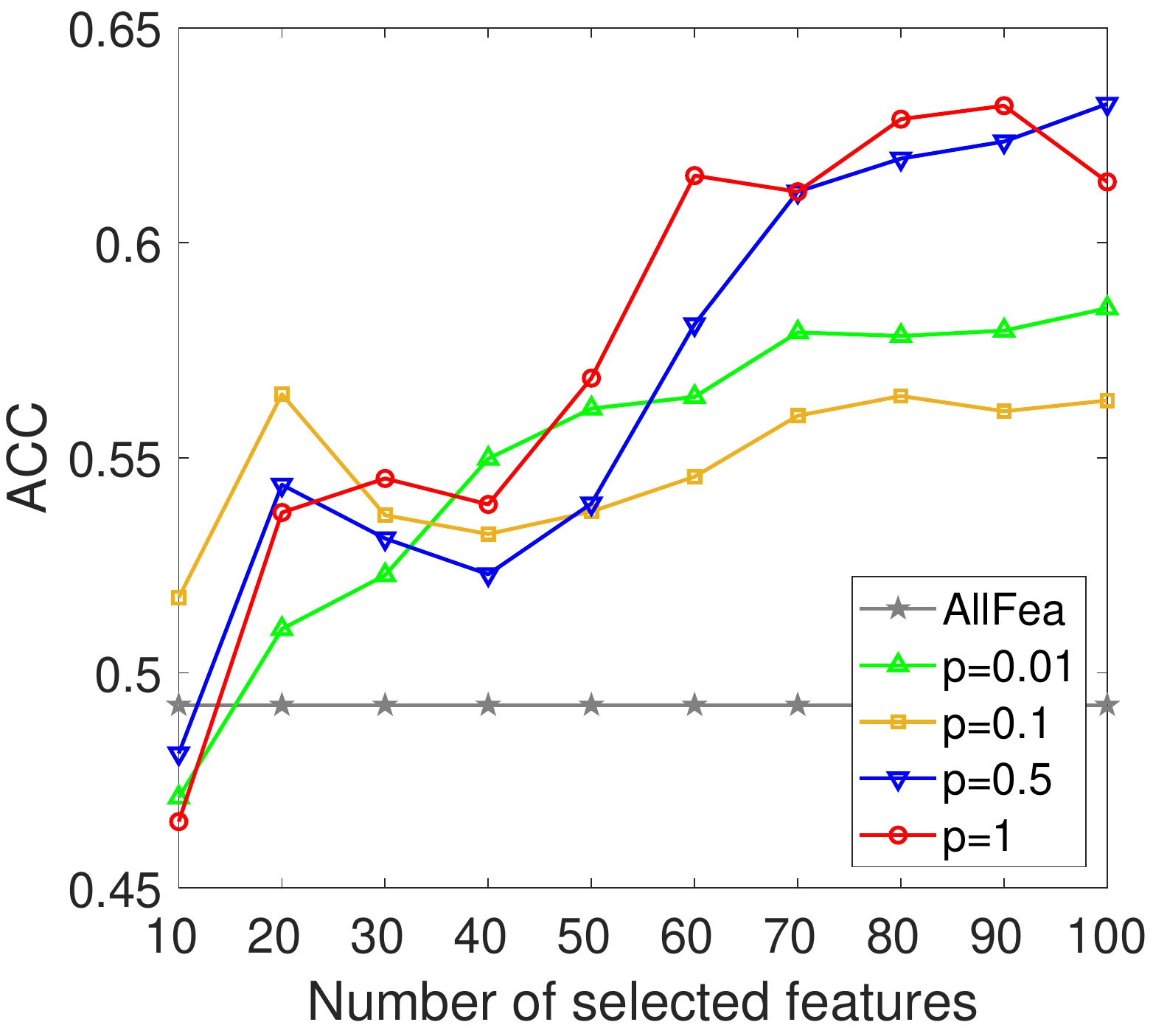}}
~\subfloat[Imm40 (NMI)]{\includegraphics[width=0.25\textwidth ]{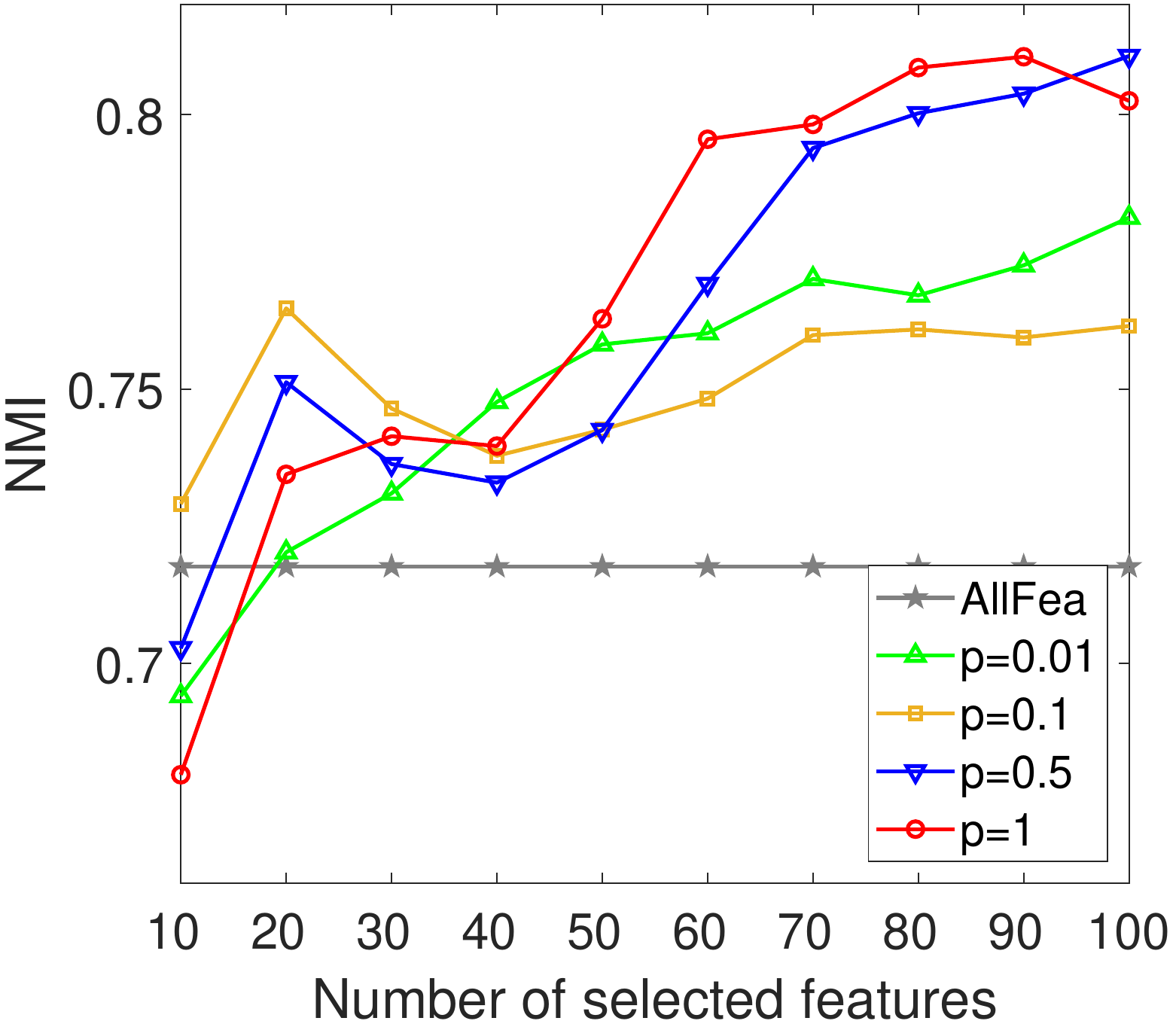}}\\
\subfloat[SRBCTML (ACC)]{\includegraphics[width=0.25\textwidth ]{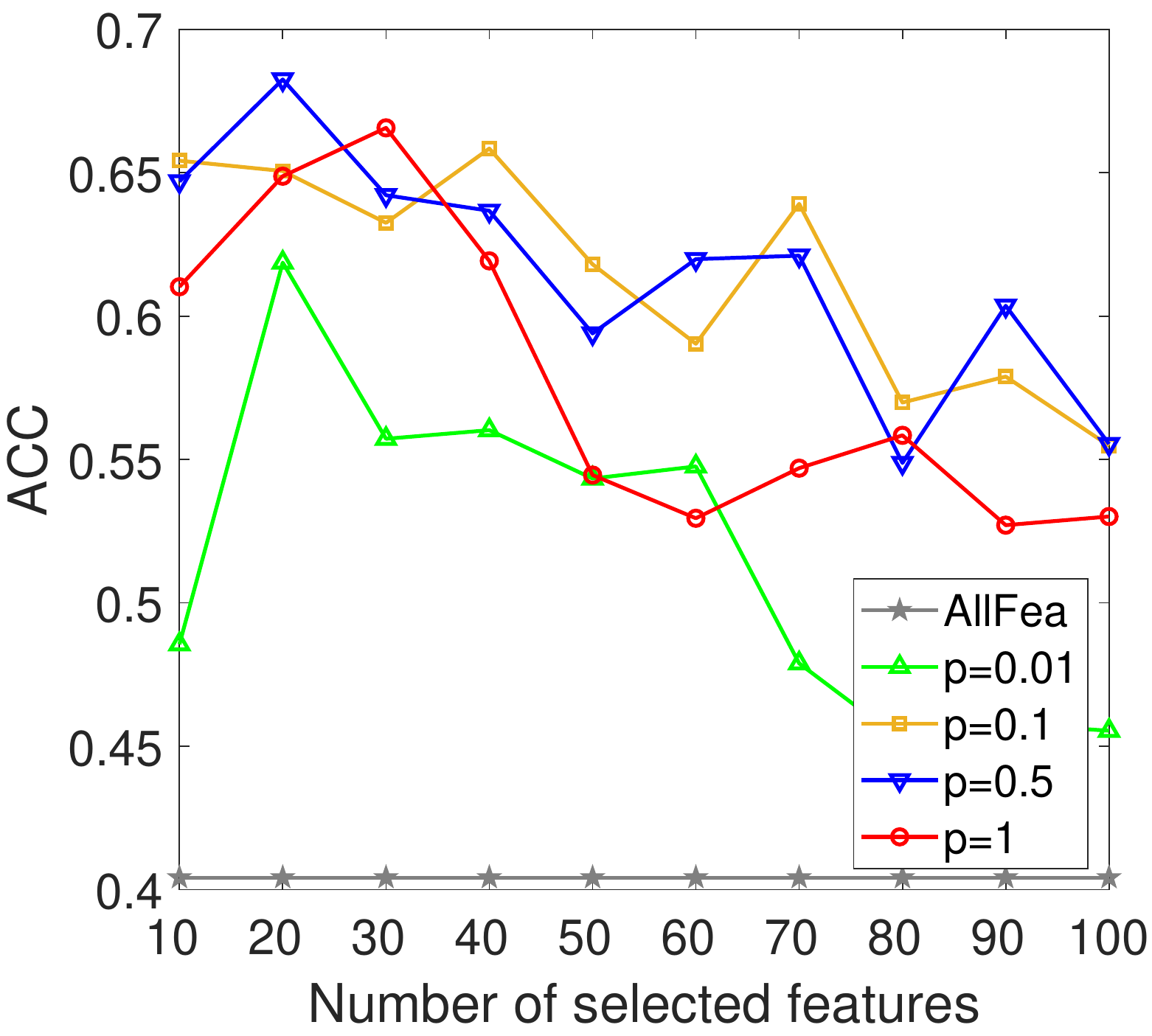}}
~\subfloat[SRBCTML (NMI)]{\includegraphics[width=0.25\textwidth ]{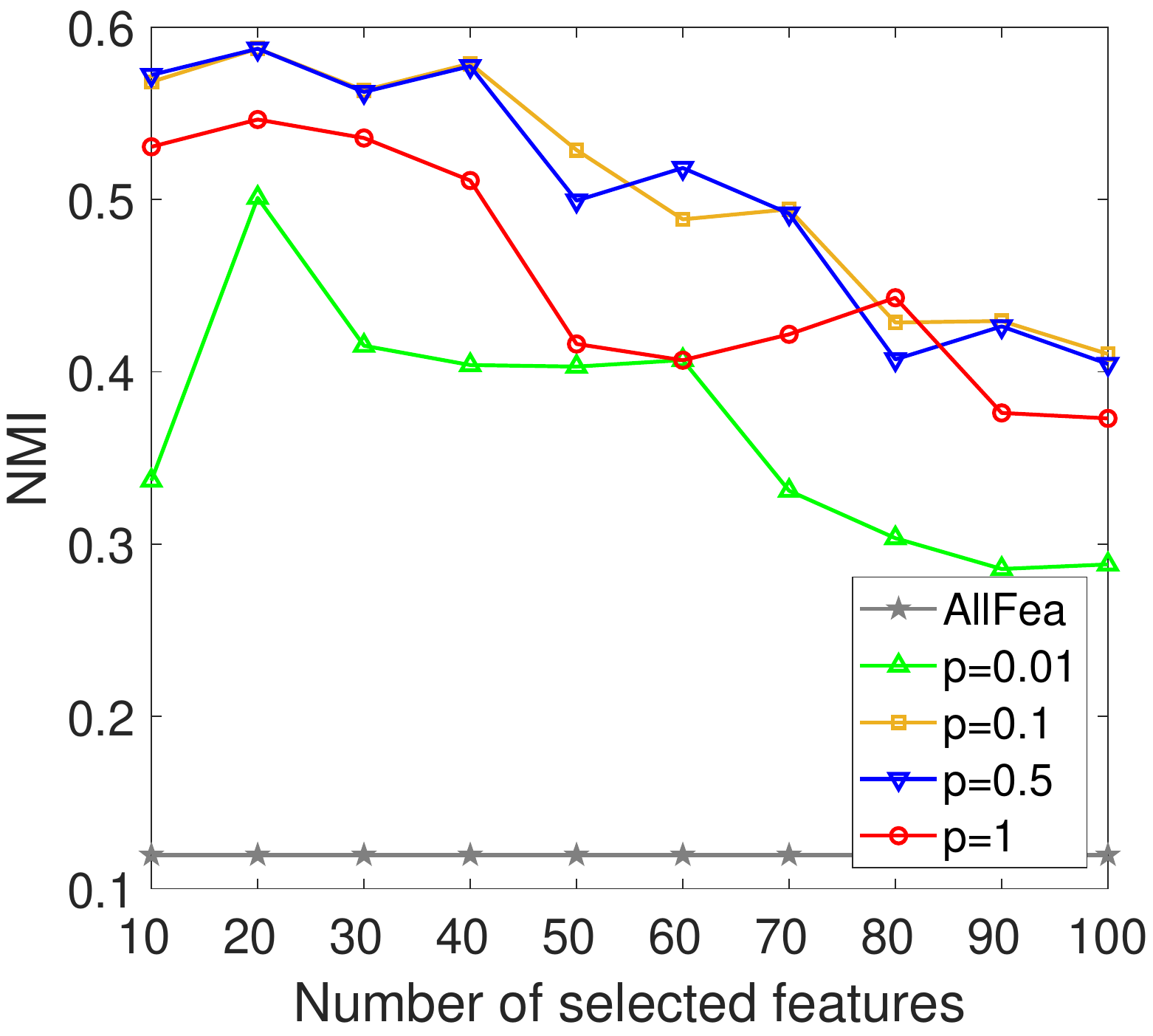}}
\caption{ACC and NMI of SPCAFS under different $p$ on Imm40, SRBCTML data sets.}
\label{fig: p analysis}
\end{figure}

On Imm40 data set,
when $p$ decreases from 1 to 0.1,
the result of feature selection of SPCAFS becomes worse.
On SRBCTML data set,
as the value of $p$ decreases from 1 to 0.01,
the results of feature selection show a trend of rising first and then falling.
The phenomenon suggests that the choice of $p$ is not the smaller the better.
The parameter $p$ is used to balance the sparsity and the convexity of the regularization.
Small $p$ results in highly non-convex problem,
which will increase the difficulty of optimization.

%

\section{Conclusions}\label{sec:conclusions}

In the paper, we propose a new method for unsupervised feature selection,
by combining reconstruction error with $\ell_{2,p}$-norm regularization.
The projection matrix is learned by minimizing the reconstruction error under the sparse constraint.
Then, we present an efficient optimization algorithm to solve the proposed unsupervised model,
and analyse the convergence and computational complexity of the proposed algorithm.
Finally,
extensive experiments on real-world data sets demonstrate the effectiveness of our proposed method.

\ifCLASSOPTIONcompsoc
  \section*{Acknowledgments}
\else
  \section*{Acknowledgment}
\fi

Thanks to the donors who have made contributions to the benchmark data sets. 

\ifCLASSOPTIONcaptionsoff
  \newpage
\fi

\bibliographystyle{IEEEtran}
\bibliography{SWC}

\begin{thebibliography}{10}
\providecommand{\url}[1]{#1}
\csname url@samestyle\endcsname
\providecommand{\newblock}{\relax}
\providecommand{\bibinfo}[2]{#2}
\providecommand{\BIBentrySTDinterwordspacing}{\spaceskip=0pt\relax}
\providecommand{\BIBentryALTinterwordstretchfactor}{4}
\providecommand{\BIBentryALTinterwordspacing}{\spaceskip=\fontdimen2\font plus
\BIBentryALTinterwordstretchfactor\fontdimen3\font minus
  \fontdimen4\font\relax}
\providecommand{\BIBforeignlanguage}[2]{{%
\expandafter\ifx\csname l@#1\endcsname\relax
\typeout{** WARNING: IEEEtran.bst: No hyphenation pattern has been}%
\typeout{** loaded for the language `#1'. Using the pattern for}%
\typeout{** the default language instead.}%
\else
\language=\csname l@#1\endcsname
\fi
#2}}
\providecommand{\BIBdecl}{\relax}
\BIBdecl

\bibitem{Ma8440052}
B.~{Ma} and A.~{Entezari}, ``An interactive framework for visualization of
  weather forecast ensembles,'' \emph{IEEE Transactions on Visualization and
  Computer Graphics}, vol.~25, no.~1, pp. 1091--1101, 2019.

\bibitem{He2020Data}
H.~He, Y.~Hong, W.~Liu, and S.~A. Kim, ``Data mining model for multimedia
  financial time series using information entropy,'' \emph{Journal of
  Intelligent and Fuzzy Systems}, no.~1, pp. 1--7, 2020.

\bibitem{Dentith2020}
M.~Dentith, R.~J. Enkin, W.~Morris, C.~Adams, and B.~Bourne, ``Petrophysics and
  mineral exploration: a workflow for data analysis and a new interpretation
  framework,'' \emph{Geophysical Prospecting}, vol.~68, no.~1, pp. 178--199,
  2020.

\bibitem{Deng2019}
C.~{Deng}, E.~{Yang}, T.~{Liu}, J.~{Li}, W.~{Liu}, and D.~{Tao}, ``Unsupervised
  semantic-preserving adversarial hashing for image search,'' \emph{IEEE
  Transactions on Image Processing}, vol.~28, no.~8, pp. 4032--4044, 2019.

\bibitem{Ali2019Fuzzy}
F.~Ali, S.~El-Sappagh, and D.~Kwak, ``Fuzzy ontology and lstm-based text
  mining: A transportation network monitoring system for assisting travel,''
  \emph{Sensors}, vol.~19, no.~2, 2019.

\bibitem{Luo8097006}
P.~{Luo}, L.~{Tian}, J.~{Ruan}, and F.~{Wu}, ``Disease gene prediction by
  integrating ppi networks, clinical rna-seq data and omim data,''
  \emph{IEEE/ACM Transactions on Computational Biology and Bioinformatics},
  vol.~16, no.~1, pp. 222--232, 2019.

\bibitem{2003An}
I.~Guyon and A.~Elisseeff, ``An introduction to variable and feature
  selection,'' \emph{Journal of Machine Learning Research}, vol.~3, no.~6, pp.
  1157--1182, 2003.

\bibitem{Zhang2019Fuzzy}
R.~Zhang, J.~Tao, and H.~Zhou, ``Fuzzy optimal energy management for fuel cell
  and supercapacitor systems using neural network based driving pattern
  recognition,'' \emph{IEEE Transactions on Fuzzy Systems}, vol.~27, no.~1, pp.
  45--57, 2019.

\bibitem{Idrobo2019clustering}
A.~{Hoyos-Idrobo}, G.~{Varoquaux}, J.~{Kahn}, and B.~{Thirion}, ``Recursive
  nearest agglomeration (rena): Fast clustering for approximation of structured
  signals,'' \emph{IEEE Transactions on Pattern Analysis and Machine
  Intelligence}, vol.~41, no.~3, pp. 669--681, 2019.

\bibitem{Kayabol2020classification}
K.~{Kayabol}, ``Approximate sparse multinomial logistic regression for
  classification,'' \emph{IEEE Transactions on Pattern Analysis and Machine
  Intelligence}, vol.~42, no.~2, pp. 490--493, 2020.

\bibitem{Wu2020retrieval}
Y.~{Wu}, S.~{Wang}, and Q.~{Huang}, ``Online fast adaptive low-rank similarity
  learning for cross-modal retrieval,'' \emph{IEEE Transactions on Multimedia},
  vol.~22, no.~5, pp. 1310--1322, 2020.

\bibitem{Yu2004Eficient}
L.~Yu and H.~Liu, ``Eficient feature selection via analysis of relevance and
  redundancy,'' \emph{Journal of Machine Learning Research}, vol.~5, no.~12,
  pp. 1205--1224, 2004.

\bibitem{Li2014Clustering}
Z.~Li, L.~Jing, Y.~Yi, X.~Zhou, and H.~Lu, ``Clustering-guided sparse
  structural learning for unsupervised feature selection,'' \emph{IEEE
  Transactions on Knowledge and Data Engineering}, vol.~26, no.~9, pp.
  2138--2150, 2014.

\bibitem{Lazar2012A}
C.~Lazar, ``A survey on filter techniques for feature selection in gene
  expression microarray analysis,'' \emph{IEEE/ACM Transactions on
  Computational Biology and Bioinformatics}, vol.~9, no.~4, pp. 1106--1119,
  2012.

\bibitem{LaplacianScore2005}
X.~He, D.~Cai, and P.~Niyogi, ``Laplacian score for feature selection,'' in
  \emph{Advances in Neural Information Processing Systems}, 2005, pp. 507--514.

\bibitem{Deng2010Unsupervised}
C.~Deng, C.~Zhang, and X.~He, ``Unsupervised feature selection for
  multi-cluster data,'' in \emph{Acm Sigkdd International Conference on
  Knowledge Discovery and Data Mining}, 2010, pp. 333--342.

\bibitem{Kabir2008A}
M.~M. Kabir, M.~M. Islam, and K.~Murase, ``A new wrapper feature selection
  approach using neural network,'' \emph{Neurocomputing}, vol.~73, no. 16-18,
  pp. 3273--3283, 2008.

\bibitem{Chong2016Feature}
P.~Chong, K.~Zhao, Y.~Ming, and C.~Qiang, ``Feature selection embedded subspace
  clustering,'' \emph{IEEE Signal Processing Letters}, vol.~23, no.~7, pp.
  1018--1022, 2016.

\bibitem{Li2016Feature}
J.~Li, K.~Cheng, S.~Wang, F.~Morstatter, R.~P. Trevino, J.~Tang, and H.~Liu,
  ``Feature selection: A data perspective,'' \emph{Acm Computing Surveys},
  vol.~50, no.~6, pp. Article 39:1--45, 2016.

\bibitem{YangUDFS2011}
Y.~Yang, H.~Shen, Z.~Ma, Z.~Huang, and X.~Zhou, ``L21-norm regularized
  discriminative feature selection for unsupervised learning,'' 07 2011, pp.
  1589--1594.

\bibitem{Wang2015Embedded}
S.~Wang, J.~Tang, and H.~Liu, ``Embedded unsupervised feature selection,'' in
  \emph{Twenty-Ninth AAAI Conference on Artificial Intelligence}, 2015, pp.
  470--476.

\bibitem{GuoDGUFS2018}
J.~Guo and W.~Zhu, ``Dependence guided unsupervised feature selection,'' in
  \emph{Proceedings of the Thirty-Second {AAAI} Conference on Artificial
  Intelligence, New Orleans, Louisiana, USA, February 2-7, 2018}.\hskip 1em
  plus 0.5em minus 0.4em\relax {AAAI} Press, 2018, pp. 2232--2239.

\bibitem{SOGFS2019}
F.~Nie, W.~Zhu, and X.~Li, ``Structured graph optimization for unsupervised
  feature selection,'' \emph{IEEE Transactions on Knowledge and Data
  Engineering}, vol.~PP, pp. 1--1, 08 2019.

\bibitem{RNE2020}
Y.~Liu, D.~Ye, W.~Li, and H.~Wang, ``Robust neighborhood embedding for
  unsupervised feature selection,'' \emph{Knowledge-Based Systems}, vol. 193,
  p. 105462, 04 2020.

\bibitem{2002Principal}
I.~T. Jolliffe, ``Principal component analysis,'' \emph{Journal of Marketing
  Research}, vol.~87, no.~4, p. 513, 2002.

\bibitem{Zhu2019accNMI}
X.~{Zhu}, S.~{Zhang}, Y.~{Li}, J.~{Zhang}, L.~{Yang}, and Y.~{Fang}, ``Low-rank
  sparse subspace for spectral clustering,'' \emph{IEEE Transactions on
  Knowledge and Data Engineering}, vol.~31, no.~8, pp. 1532--1543, 2019.

\bibitem{Strehl2002Cluster}
A.~Strehl and J.~Ghosh, ``Cluster ensembles: a knowledge reuse framework for
  combining partitionings,'' \emph{Journal of Machine Learning Research},
  vol.~3, no.~3, pp. 583--617, 2002.

\end{thebibliography}
\end{document}